\documentclass[preprint,review]{elsarticle} 
\usepackage{algorithm}
\usepackage{algorithmic}
\usepackage{caption}
\usepackage{subcaption}
\newcommand{\eg}{\textit{e.g.}}
\newcommand{\ie}{\textit{i.e.}}
\usepackage{amsthm}
\usepackage{amsmath}
\newtheorem{theorem}{Theorem}
\usepackage{multirow}
\usepackage{amssymb}
\newtheorem{myDef}{Definition}
\usepackage{newfloat}
\usepackage{listings}
\usepackage{url} 
\usepackage{lscape} 
\usepackage{colortbl}

\usepackage{xcolor}

\usepackage{soul}
\soulregister{\ref}7
\soulregister{\cite}7
\soulregister{\citet}7

\journal{Neural Networks}

\begin{document}
\begin{frontmatter}

\title{Harnessing Collective Structure Knowledge in Data Augmentation for Graph Neural Networks}

\author[1]{Rongrong Ma}
\ead{Rongrong.ma-1@student.uts.edu.au}
\author[2]{Guansong Pang\corref{cor1}}
\ead{gspang@smu.edu.sg}
\author[3]{Ling Chen}
\ead{ling.chen@uts.edu.au}
\cortext[cor1]{Corresponding author}
\affiliation[1]{organization={Faculty of Engineering and Information Technology, University of Technology Sydney},
             addressline={123 Broadway},
             city={Sydney},
             postcode={2007},
             state={NSW},
             country={Australia}}

\affiliation[2]{organization={School of Computing and Information Systems, Singapore Management University},
             addressline={80 Stamford Rd},
             postcode={178902},
             country={Singapore}}
             
\affiliation[3]{organization={Faculty of Engineering and Information Technology, University of Technology Sydney},
             addressline={123 Broadway},
             city={Sydney},
             postcode={2007},
             state={NSW},
             country={Australia}}

\begin{abstract}
Graph neural networks (GNNs) have achieved state-of-the-art performance in graph representation learning.
Message passing neural networks, which learn representations through recursively aggregating information from each node and its neighbors, are among the most commonly-used GNNs. However, a wealth of structural information of individual nodes and full graphs is often ignored in such process, which restricts the expressive power of GNNs. Various graph data augmentation methods that enable the message passing with richer structure knowledge have been introduced as one main way to tackle this issue, but they are often focused on individual structure features and difficult to scale up with more structure features. In this work we propose a novel approach, namely \underline{co}llective \underline{s}tructure knowledge-augmented \underline{g}raph \underline{n}eural \underline{n}etwork (CoS-GNN), in which a new message passing method is introduced to allow GNNs to harness a diverse set of node- and graph-level structure features, together with original node features/attributes, in augmented graphs. In doing so, our approach largely improves the structural knowledge modeling of GNNs in both node and graph levels, resulting in substantially improved graph representations.
This is justified by extensive empirical results where CoS-GNN outperforms state-of-the-art models in various graph-level learning tasks, including graph classification, anomaly detection, and out-of-distribution generalization. Code is available at: \url{https://github.com/RongrongMa/CoS-GNN}.
\end{abstract}

\end{frontmatter}

\section{Introduction}
Graph representation learning is one of the most popular topic in graph mining because of its numerous potential applications in bioinformatics~\cite{borgwardt2005protein,borgwardt2005shortest,zhao2022deep,song2022graph,wu2023survey}, chemical~\cite{hassani2022cross,li2022geomgcl,zhang2022graph}, social networks~\cite{zhang2022improving,Stankovic2020data,yanardag2015deep,shen2023uniskgrep} and cyber security~\cite{he2022illuminati}. In the past few years, Graph Neural Networks (GNNs)~\cite{wu2020comprehensive,chikwendu2023comprehensive} have been emerging as one of the most powerful and successful techniques for graph representation learning. 

Message passing neural networks constitute a prevalent category of GNN models, which learn node features and graph structure information through recursively aggregating current representations of node and its neighbors. Diverse aggregation strategies have been introduced, giving rise to various GNN backbones, such as GCN, GIN, and among others~\cite{wu2020comprehensive,kipfsemi,hamilton2017inductive,velivckovicgraph,xupowerful}. However, the expressive power of these message passing GNNs is upper bounded by 1-dimensional Weisfeiler-Leman (1-WL) tests~\cite{xupowerful,morris2019weisfeiler} that encode a node’s color via recursively expanding the neighbors of the node to construct a rooted subtree for the node.
% which has limited performance
As shown in Figure \ref{fig:WLfail}, such rooted subtrees are with limited expressiveness and might be the same for graphs with different structures, leading to failure in distinguishing these graphs. This presents a bottleneck for applying WL tests or message passing neural networks to many real-world graph application domains.

The failure of WL test is mainly due to the rooted subtree's limited capabilities in capturing different substructures that can appear in the graph.
Since the message passing scheme of GNNs mimics the 1-WL algorithm, one intuition to enhance the expressive power of GNNs is to enrich the passing information, especially structural knowledge, to help GNNs model diverse substructures. One popular approach to achieve this is data augmentation (DA) techniques~\cite{ding2022data}. 
% Graph data augmentation normally includes operations on node features or graph structures. 
% \citet{cui2022positional} points out that proper artificial node features can be as informative as real features. 
One general framework in this line is to compute additional node features based on structural properties and attach them to original node features, such as DE~\cite{li2020distance}, GSN~\cite{bouritsas2022improving}, fast ID-GNN~\cite{you2021identity} and LAGNN~\cite{liu2022local}. 
%One popular augmentation approach is to generate and include into original node features, such as DE~\cite{li2020distance}, GSN~\cite{bouritsas2022improving}, fast ID-GNN~\cite{you2021identity} and LAGNN~\cite{liu2022local}.
% attach extra information to the original node features. 
%The distance encoding is employed as the extra node features in \cite{li2020distance}. GSN~\cite{bouritsas2022improving} and a variant of ID-GNN~\cite{you2021identity} extend the node features with the number of different motifs. \cite{cui2022positional} points out that proper artificial node features can be as informative as real features. Liu et al.~\cite{liu2022local} proposes a local augmentation method, that augments the features of nodes via the representation distribution of the node's neighbors. 
Except extending node features, NestedGNN~\cite{zhang2021nested} and ID-GNN~\cite{you2021identity} compute and add node embeddings based on the local subgraph of each node. However, these methods only focus on local structure while many important global structure features are ignored. Also, GSN and fast ID-GNN often rely on a properly pre-defined substructure set to incorporate domain-specific inductive biases~\cite{zhang2021nested}. Further, these DA techniques are focused on augmenting the graph with some individual features, which are difficult to scale up to the incorporation of a diverse, large set of augmented features.

\begin{figure}
    \centering
    \includegraphics[width=8cm]{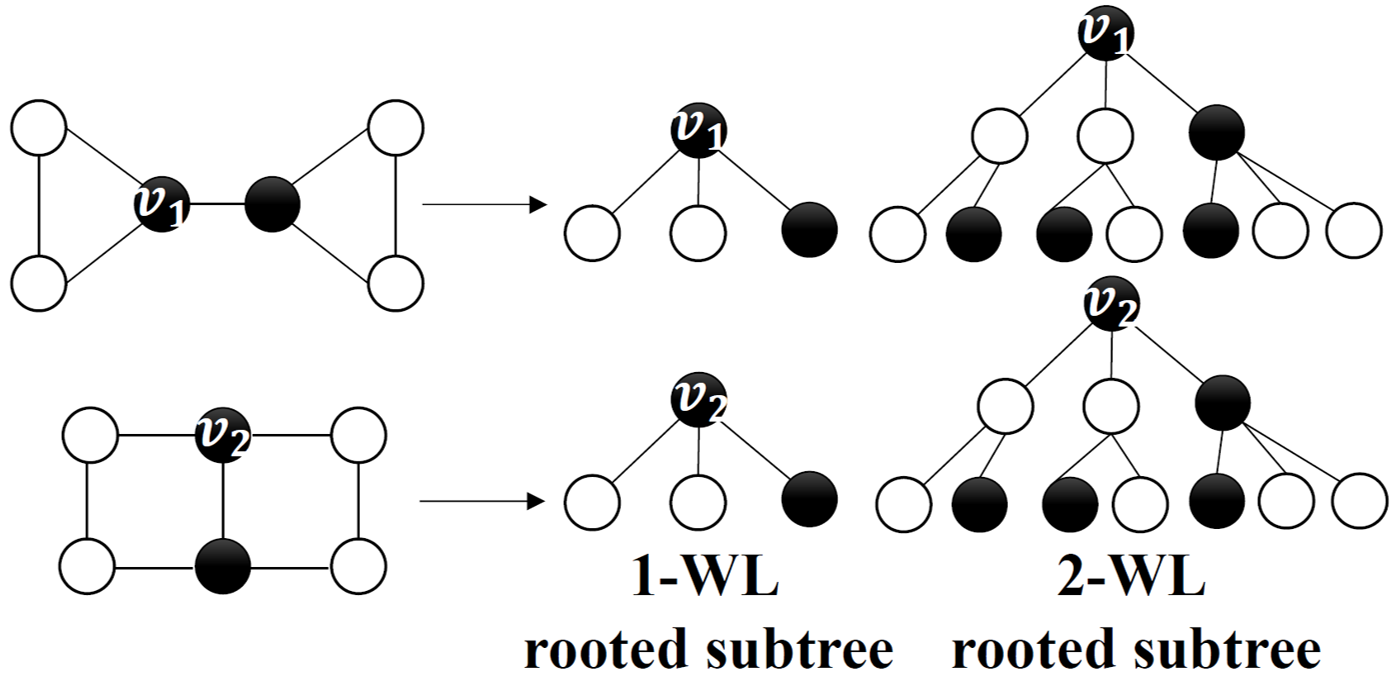}
    \caption{1- and 2-WL tests fail to distinguish the two graphs as they obtain the same rooted subtree (node coloring).}
    \label{fig:WLfail}
\end{figure}

%Another line of this research is to disturb graph structure. For instance, DropGNN~\cite{papp2021dropgnn} increases the expressiveness of GNNs by executing multiple different runs on node-dropout graphs. $\mathcal{G}$-mixup~\cite{han2022g} and graph transplant~\cite{park2022graph} perform mixup operations over graphs to obtain more graph data. \cite{liu2022boosting} insert a dummy node in each graph while retaining the original node and edge properties for better graph representation learning. 

%Although these methods try to exploit graph structure as a great role in the graph representation learning process, the implicit use of structure information in most methods only can provide limited knowledge. The explicit structural information is used in \cite{cui2022positional}, GSN and fast ID-GNN. However, \cite{cui2022positional} just tries to use single artificial features for non-attributed graphs and proves that some artificial features can be same efficient with real features. The problem that might restrict the results in GSN and fast ID-GNN is that these methods need a properly defined substructure set to incorporate domain-specific inductive biases~\cite{zhang2021nested}. Besides, they only focus some specific substructure while many informative structural characteristics are ignored. \guansong{pls rewrite this paragraph. It needs to connect to 1WL test, e.g., how it addresses this test, or what thing is missing. We need to have a smooth story line}

% Inspired by previous graph augmentation methods,
In this work, we propose a novel approach, namely \textit{\underline{co}llective \underline{s}tructure knowledge-augmented \underline{g}raph \underline{n}eural \underline{n}etwork (CoS-GNN)}, to leverage a variety of informative structural knowledge of graphs through DA for enhancing the expressiveness of existing GNNs. Instead of implicitly using structural information in other DA methods, we explicitly extract collective, domain-adaptive graph structural statistics at the graph and node levels as additional structure features.
% and utilize them in a specially designed message passing process. To avoid the dependence on domain knowledge and substructure size chosen in previous substructure counting based augmentation methods, we calculate some specifically scale-adaptive substructural statistics for single node and whole graph. Besides, statistics that can quantify the importance of nodes in the graph and the effectiveness of communication in the graphs are also employed. These statistics are divided into augmented node features and graph features. 
To fully leverage those augmented structural knowledge, we design a new message passing mechanism to respectively perform neighborhood aggregation on graph data using these augmented structure features and the original node attributes. Further, the new message passing can also model the interaction between the augmented features and the original node attributes. In doing so, our GNNs break down the upper bound of 1-WL tests and learn graph representations with significantly improved expressiveness (see the graph representations produced by CoS-GNN in Figure \ref{fig:motivation}(b)(c) vs. those yielded by the original GCN).
% During the node message passing process, these extra structural information can provide more abundant information for each node and the real node attributes and the augmented structural statistics are regarded as supplementary for each other to enrich their learned knowledge. Later, the artificially augmented graph structural features are combined with the the learned graph representations to represent the graph. As can be seen from Figure~\ref{fig:motivation}, the additional structure features can provide more discriminative information and make the learned graph representation more expressive and separate.  

\begin{figure}
    \centering
    \includegraphics[width=12cm]{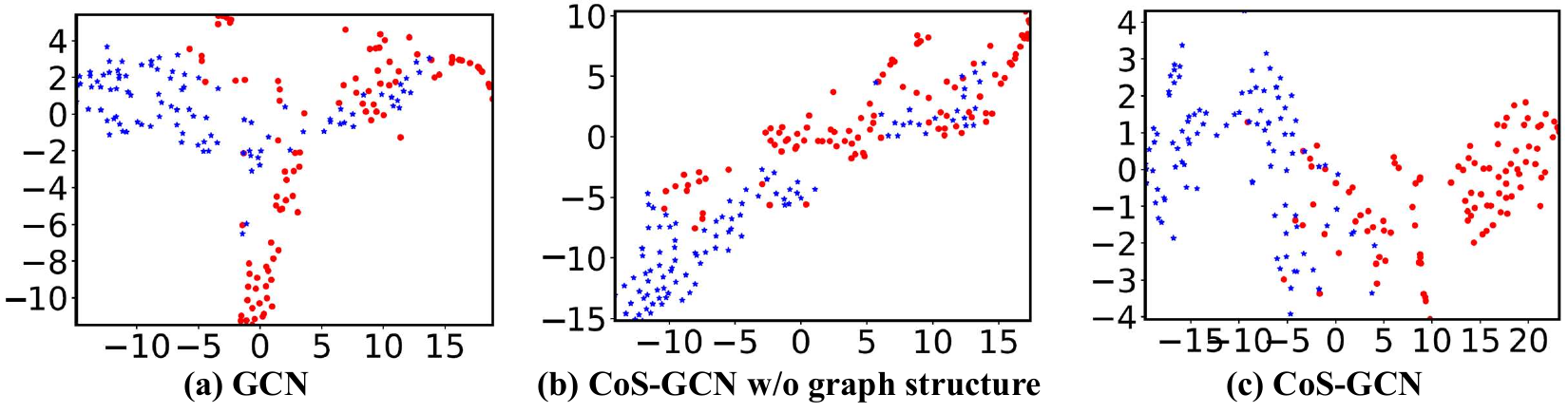}
    \caption{Graph representations of REDDIT-BINARY yielded by (b) CoS-GCN with augmented node-level structural features and (c) CoS-GCN with augmented structural features at both node and graph levels are more class-separable than those produced by (a) the original GCN.}
    \label{fig:motivation}
\end{figure}

In summary, our main contributions are as follows:
% makes the following three main contributions:
\begin{itemize}
    \item We introduce a novel collective structure knowledge augmented GNN approach (CoS-GNN) that explicitly harnesses a diverse set of node and graph-level structural information for enhancing the expressiveness of GNN-based graph representations. The approach is generic and applicable to different GNN backbones.
    \item To effectively leverage the augmented structural features, a new message passing scheme is introduced in CoS-GNN, which simultaneously performs neighborhood aggregation on the augmented features and the original node attributes, enabling the learning of graph representations with significantly enriched structural knowledge.
    % a sufficient fusion of the augmented information and the real features.   
    \item Comprehensive experiments on 12 graph datasets demonstrate that CoS-GNN (i) significantly outperforms competing methods in graph classification task; 
    %(ii) offers a GNN backbone that can be combined with other GNN-based frameworks and other pooling methods; 
    (ii) provides more discriminative information for anomaly detection task; (iii) is more generalized to out-of-distribution graphs.
\end{itemize}

\section{Related Work}
\subsection{Graph Neural Networks}
Graph neural networks (GNNs) have gained numerous attentions and remarkable success in the past few years~\cite{wu2020comprehensive, ju2022kgnn,luo2023towards,chikwendu2023comprehensive}. \citet{wu2020comprehensive} summarize various GNN models, among which message passing between nodes based on graph structure to learn graph representations is one of the most popular ones. They iteratively aggregate information from nodes and their neighbors to learn expressive representations of nodes. %That is, the representation of node $v$ after $l_{th}$ layer in a GNN message-passing module can be formulated as follows:
%\begin{equation}
%    \begin{aligned}
%        \mathbf{h}_v^{(l)} = \mathrm{COM}(\mathbf{h}_v^{(l-1)},\mathrm{AGGATE}(\{\mathbf{h}_u^{(l-1)}|u\in\mathcal{N}_v\})),
%    \end{aligned}
%\end{equation}
%where $\mathrm{COM}(\cdot)$ and $\mathrm{AGGATE}(\cdot)$ represent COMBINE and AGGREGATE functions respectively and $\mathcal{N}_v$ denotes the neighbor set of node $v$. 
GCN~\cite{kipfsemi},   
GraphSAGE~\cite{hamilton2017inductive}, graph attention network (GAT)~\cite{velivckovicgraph} and graph isomorphism network (GIN)~\cite{xupowerful} are several most representative and state-of-the-art message-passing GNNs. However, since the message passing mechanism of GNNs mimics the 1-WL algorithm, the expressive power of these GNNs is limited and upper bounded by the Weisfeiler-Leman test~\cite{xupowerful,morris2019weisfeiler}, which restricts their performance in graph representation learning. The proposed CoS-GNN aims to improve the expressiveness of these popular GNNs by modeling a collective set of additional structural features.

\subsection{Graph Data Augmentation}
Motivated by the excellent achievements of data augmentation in image and text data~\cite{shorten2019survey,bayer2022survey}, graph data augmentation techniques are attracting increasing attention to improve the representation expressiveness obtained from GNNs in graph-domain tasks~\cite{ding2022data}. To solve the label scarcity in semi/un-supervised classification tasks, augmentation through changing the graph structure or node attribute, \eg, node/edge deletion, attribute masking, and subgraph sampling, is used to construct a contrastive or consistent learning framework in \cite{ju2022ghnn,ju2023unsupervised,luo2022clear,luo2023self,ju2023tgnn}. In supervised graph classification tasks, some augmentation methods are implemented to enhance the representation expressiveness of GNNs and further improve the classification performance. 
% Because of the complexity and the scarcity of useful graph information in graph samples. %Graph augmentation techniques can be divided into three classes according to the augmentation modality, \eg structure-oriented (distrubing graph structures), feature-oriented (transforming or extending node features), and label-oriented (enriching labels) techniques~\cite{ding2022data}. 
%A significant number of augmentation algorithms are proposed for solving specific problems in graph learning, such as feature or structure denoising or over-smoothing~\cite{kazi2022differentiable,luo2021learning,wang2021graph,chen2020iterative,zhao2021heterogeneous,zhao2021data,liu2021graph,he2022analyzing,topping2021understanding,zeng2021decoupling,zhao2021adaptive}. 
%There are also some works aiming to promoting the expressiveness of GNNs while trusting the information of graphs, which is the same as our goal. 
One direction is to enrich node features. 
For example, distance encoding is proposed in \cite{li2020distance} to generate and add extra node features.
% for some structure-related tasks. 
\citet{sato2021random} add random node features, while GSN~\cite{bouritsas2022improving} and fast ID-GNN~\cite{you2021identity} use the count of various motifs to extend the node features. In \cite{liu2022local}, neighborhood features are augmented via a generative model conditioned on local structures and node features.
%the representation distribution of a target node's neighbors is learned conditioned on the representation of central node and helps to augment the features of nodes with few neighbors. 
Except to node feature extension, augmenting the graph from the structure perspective is another popular approach. For example, \citet{zhang2021nested} and \citet{you2021identity} sample a subgraph for each node and use the subgraphs to compute node embeddings and add them to complement the original graph. The structural information used in these methods is local, while many useful graph-level structural information is ignored. To alleviate this issue, 
\citet{liu2022boosting} add a dummy node that connects to all existing nodes without affecting original node and edge properties for better graph representation learning. $\mathcal{G}$-mixup~\cite{han2022g} and graph transplant~\cite{park2022graph} mixup graphs to obtain more graph data for data-hungry tasks. 
\citet{papp2021dropgnn} instead perform augmentation through iteratively removing nodes randomly and executing multiple different runs on these node-dropout graphs. All these DA methods are focused on utilizing some specific types of additional features, like substructure and local structure variation, which often cannot generalize to graphs from different application domains. By contrast, our approach aims to harness collective, domain-adaptive node and graph features via a new message passing mechanism.
% , which is time costly. 
% We also augment node features to improve the GNN's expressiveness in this work. However, we extract collective local and global structural information, which is domain-independent. 

\section{Framework}
\subsection{Problem Statement}\label{subsec:problem}
This work focuses on the problem of graph representation learning. Specifically, given a set of graphs $\mathcal{G}=\{G_1, \cdots, G_N\}$, each graph $G=\{\mathcal{V}_G,\mathcal{E}_G\}$ contains a vertex/node set $\mathcal{V}_G$ and an edge set $\mathcal{E}_G$. The structure of $G$ is denoted by an adjacency matrix $A\in \mathrm{R}^{|\mathcal{V}_G|\times |\mathcal{V}_G|}$, where $|\mathcal{V}_G|$ is the number of nodes in $G$. $A(i,j)=1$ if an edge exists between node $v_i$ and $v_j$ ($\exists (v_i,v_j)\in \mathcal{E}_G$), and $A(i,j)=0$ otherwise. If a feature vector $\mathbf{x}_i^n\in\mathrm{R}^d$ is associated with each node $v_i\in\mathcal{V}_G$, the graph $G$ is an attributed graph, and otherwise $G$ is a plain graph. For a plain graph, we use the one-hot node label as the node attributes. Our goal is to learn a representation for each graph $G$ for further use in down-stream tasks like graph classification and anomaly detection. %In this paper, we consider the graph classification task as the down-stream task, \ie we aim to learn a mapping function $f(W)\colon \mathcal{G}\rightarrow \mathcal{Y}$, where $\mathcal{Y}$ is the set of classes, such that $f(G;W)\rightarrow y$ if $G$ belongs to class $y$.  

\begin{figure}
    \centering
    \includegraphics[width=12cm]{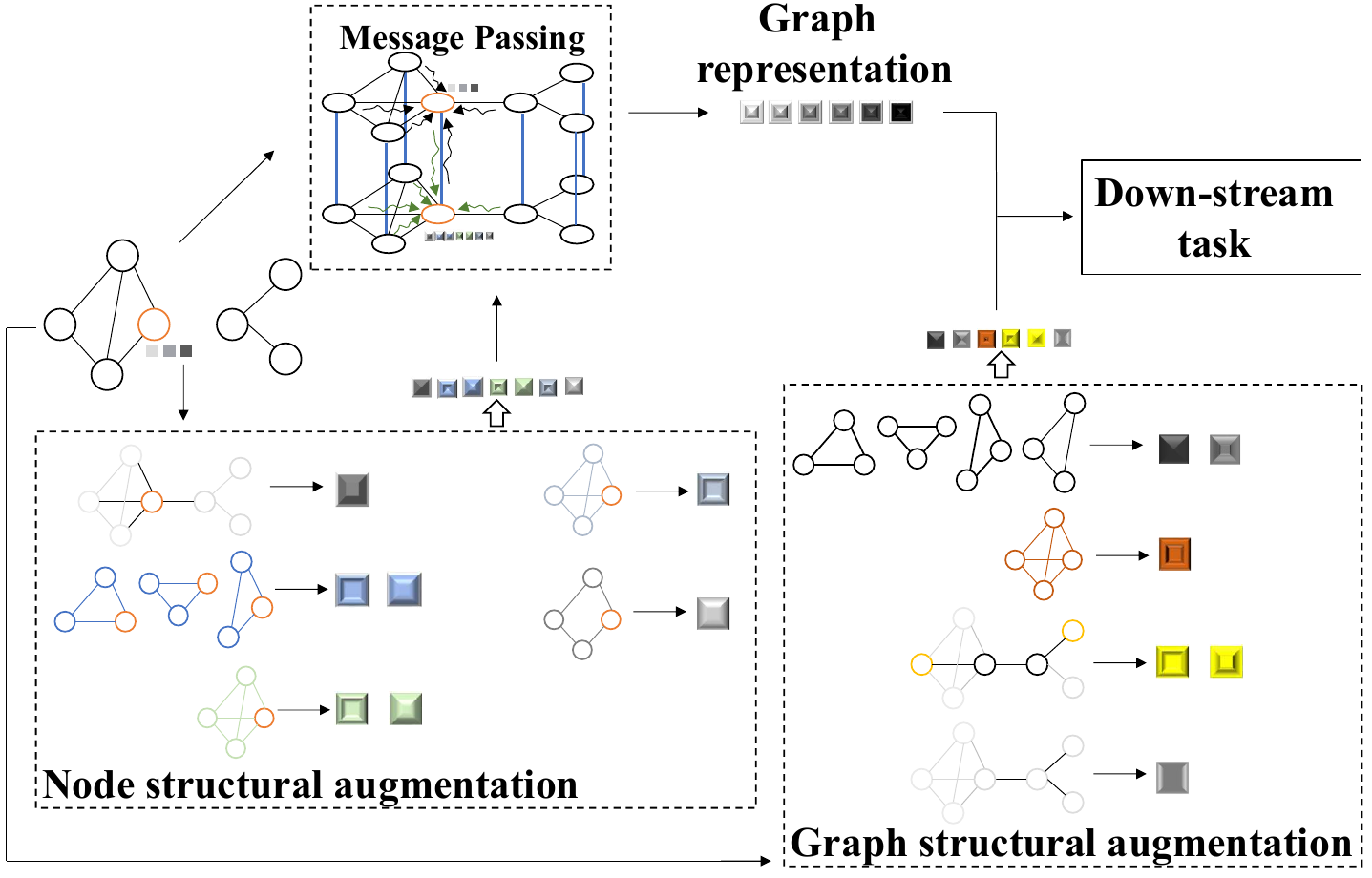}
    \caption{A schematic depiction of our CoS-GNN. Our CoS-GNN first calculates the specific node- and graph-level structural features. Then a new message passing mechanism is devised to utilize the original node attributes and the augmented node structural features to compute the graph representation, which is further combined with the graph structural augmentations for down-stream tasks.}
    \label{fig:framework}
\end{figure}

\subsection{The Proposed Framework}
We propose a novel collective structure knowledge-augmented graph neural network (CoS-GNN) that aggregates original and augmented structural features of single nodes and whole graph to learn expressive graph representations. The key intuition of CoS-GNN is to utilize various local (node) and global (graph) structural information to enrich the original graph structural knowledge, through which we can learn a more informative and discriminative graph representation. The overall procedure of CoS-GNN is illustrated in Figure~\ref{fig:framework}, which is composed of the following three major components:
\begin{itemize}
    \item \textit{Collective Graph Data Augmentation}.  In this component, we generate a diverse set of specific structural features for each graph $G$ (denoted by $\mathbf{x}^{gs}_G$) and each node $v_i$ in $G$ (denoted by $\mathbf{x}^{ns}_i$). These two types of features are added to augment each graph $G$ from the structural knowledge perspective. It is a component that can be done offline.
    \item \textit{Augmented Node-level Message Passing.} This component is designed to iteratively aggregate both the original and augmented node features, \ie, $\mathbf{x}_i^{n}$ and $\mathbf{x}^{ns}_i$, to learn the node representation $\mathbf{h}_i$ with significantly enriched structural knowledge for each node $v_i$. To this end, a new  message passing mechanism is introduced for this process. The node representations are then fed to a readout layer to gain the graph representation $\mathbf{h}_l$.
    \item \textit{Graph-level Representation Fusion.} This component aims to synthesize the learned graph representation $\mathbf{h}_l$ and the pre-defined graph-level structural features $\mathbf{x}^{gs}_G$ via concatenation/fully-connected layers to obtain the final representation $\mathbf{h}_g$. $\mathbf{h}_g$ is then fed to a down-stream graph-level learning task. 
\end{itemize}

\section{Model} \label{sec:model}
CoS-GNN is a generic framework. In this section, we introduce two instantiations of our CoS-GNN framework with the commonly-used GCN and GIN as the GNN backbone, namely CoS-GCN and CoS-GIN, respctively.

\subsection{Collective Graph Data Augmentation}
We first augment the graph via computing some important node and graph statistics, which serve as additional node and graph features to complement the original node attributes. This component is shared by different model instantiations, and it can be performed before the model training.

Specifically, we select and generate a number of widely-used and domain-adaptive node-level features, including the degree, triangle number, clique size, clique number, core number, cluster coefficient and square cluster coefficient,
% to form the node-structural-augmented statistics
resulting seven new features in $\mathbf{x}^{ns}_i$ for each node $v_i$. The last two coefficient measures capture the tendency of the
node to form relatively dense communities, while other measures are to capture substructural information from varying scales. The detailed definition of these features is as follows:
\begin{itemize}
    \item \textbf{Degree.} The degree of a node/vertex is the number of edges that are incident to the node, which is an important and commonly-used node structure statistic.
    \item \textbf{Triangle.} Triangle is a simple and direct structure, and we counts the number of triangles that use this node as a vertex. 
    \item \textbf{Clique.} The clique is a substructure, in which every two distinct nodes are adjacent. We calculate the size of the maximal clique and the number of maximal cliques containing each given node.
    \item \textbf{K-core.} A k-core is defined as a maximal subgraph that is composed of nodes with degree k or more, the core number of a node is the largest value k of a k-core containing the given node. We collect the core number of each node as one of the augmented node-structural characteristics. 
    \item \textbf{Quantized values.} Beyond the number, we also calculate the triangle/square clustering coefficient for each node, which are the fraction of possible triangles/squares through the given node that exist. This quantifies the tendency of nodes to form relatively dense network groups, \ie, triangles or squares. 
\end{itemize}

For graph-structural-level augmentation, we utilize a variety of important global statistics, including triangle number, clique size, the existence of bridge, average clustering coefficient, average global efficiency, and average local efficiency, to generate six graph-level structural features $\mathbf{x}^{gs}_{G}$ for each graph $G$. The three coefficients quantify the abundance of dense communities in the graph and the other statistics are the measurement of the node-to-node communication effectiveness within a graph. Detailed definition of each statistic is presented as follows:
\begin{itemize}
    \item \textbf{Triangle.} We use the total number of triangles as one graph feature.  
    \item \textbf{Clique.} We count the size of the largest clique in the graph as the second graph feature.
    \item \textbf{Bridge.} Another employed statistic is the existence of a bridge in the graph, which is an edge whose removal will cause the number of connected components of the graph to increase. The bridge is a specific characteristic of the graph. 
    \item \textbf{Quantized values.} The average clustering coefficient for the graph is also included to measure the abundance of dense network groups in the graph. The efficiency of a pair of nodes is the multiplicative inverse of the shortest path distance between the nodes, and we calculate the average efficiency of all pairs of nodes in the graph, called average global efficiency, as one of the graph-structural statistics to measure the effectiveness of communication in the graph. The local efficiency of a node is defined as the average global efficiency of the subgraph induced by the neighbors of the node. We utilize the average local efficiency, which is the mean of local efficiencies of each node in the graph, as another statistic. 
\end{itemize}

These collective statistics consider the configurations with different scales and complexities, which are normally adaptive to graphs from different domains.

\subsection{Augmented Node-level Message Passing}
Once the augmented node structural feature $\mathbf{x}^{ns}$ is obtained, we then aggregate the original feature $\mathbf{x}^{n}$ and the augmented features $\mathbf{x}^{ns}$ to learn the original node attributes and their interaction with augmented structural knowledge of nodes. One straightforward solution that many previous methods do is to concatenate them directly and then apply GNN to perform the commonly-used neighborhood aggregation on nodes using the combined feature. This approach is easy-to-implement but fails to capture intricate interactions (e.g., higher-order and/or non-linear interactions) between the original node attributes and augmented features.
% We avoid doing the simple concatenation
To address this issue, we propose a novel message passing mechanism for effectively capturing the diverse knowledge embedded in the two types of features and their interactions. Our experiments also show that our message passing mechanism outperforms the conventional message passing with the concatenated input (see results in Table \ref{ablation}). 

To this end, we construct a dual-graph structure that facilitates the modeling of the original node features, the modeling of the collective augmented node features, and the modeling of the interactions between these two types of features in each message passing step. In detail, given a graph $G$, we construct a new graph $\hat{G}$ with the same node and structure as the original graph but with the $\mathbf{x}^{ns}$ as its node attributes and link the corresponding nodes of $G$ and $\hat{G}$. This results in our augmented graph with a dual-graph structure, $G^{'}$. 

%Thus, in the $l_{th}$ layer of the message-passing module, the representation of node $v$ has two parts:
%\begin{equation}
%    \begin{aligned}
%        \mathbf{h}_{v,n}^{(l)} \!=\! \mathrm{COM}(\mathbf{h}_{v,n}^{(l-1)},\mathrm{AGGATE}(\{\mathbf{h}_{u,n}^{(l-1)}|u\!\in\!\mathcal{N}_v\}),\mathbf{h}_{v,ns}^{(l-1)}),
%    \end{aligned}
%\end{equation}
%\begin{equation}
%    \begin{aligned}
%        \mathbf{h}_{v,ns}^{(l)} \!=\! \mathrm{COM}(\mathbf{h}_{v,ns}^{(l-1)},\mathrm{AGGATE}(\{\mathbf{h}_{u,ns}^{(l-1)}|u\!\in\!\mathcal{N}_v\}),\mathbf{h}_{v,n}^{(l-1)}),
%    \end{aligned}
%\end{equation}
%where $\mathbf{h}_{v,n}^{(l)}$ and $\mathbf{h}_{v,ns}^{(l)}$ represent the learned node embeddings with original feature and augmented structural feature respectively and the node embedding of $l_{th}$ layer can be obtained by concatenating the two parts, \ie $\mathbf{h}_{v}^{(l)}=\mathrm{CONCATE}(\mathbf{h}_{v,n}^{(l)},\mathbf{h}_{v,ns}^{(l)})$.

\subsubsection{Message Passing in CoS-GCN} Next we perform message passing on the dual-graph structure $G^{'}$. When using GCN as our GNN backbone, the adjacent matrix $A^{'}$ of $G^{'}$ can be written as 
\begin{equation}
\begin{aligned}
   A^{'} = 
    \begin{pmatrix}
        A & I \\
        I & A
    \end{pmatrix} ,
\end{aligned}
\end{equation}
and the degree matrix $D^{'}$ is 
\begin{equation}
\begin{aligned}
    D^{'} = 
    \begin{pmatrix}
        D+I & 0 \\
        0 & D+I
    \end{pmatrix},
\end{aligned}
\end{equation}
where $A$ and $D$ are the adjacent and degree matrices of $G$ and $I$ is the identity matrix. We then convolute the node features of $G^{'}$ by 
\begin{equation}\label{eq:convolution}
\small{
\setlength{\arraycolsep}{1.2pt}
\begin{aligned}
& H^{(l)} =\sigma\left( \tilde{D}^{'-\frac{1}{2}} \tilde{A}^{'} 
 \tilde{D}^{'-\frac{1}{2}}
 \begin{pmatrix}
      H_n^{(l-1)}\\
      H_{ns}^{(l-1)}
 \end{pmatrix}W^{(l)}\right) \\
& = \sigma \left(\begin{pmatrix}
     \tilde{D} + I &0 \\
     0 & \tilde{D} + I
 \end{pmatrix}^{-\frac{1}{2}}
 \begin{pmatrix}
        \tilde{A} & I \\
        I & \tilde{A}
    \end{pmatrix}
    \begin{pmatrix}
     \tilde{D}+I &0 \\
     0 & \tilde{D}+I
 \end{pmatrix}^{-\frac{1}{2}}
 \begin{pmatrix}
     H_n^{(l-1)}W^{(l)} \\
     H_{ns}^{(l-1)}W^{(l)}
 \end{pmatrix}\right) \\
% & \!\!=\!\!\sigma\!\!\left(\!\!\!\begin{pmatrix}
%      (\tilde{D}\!\!+\!\!I)^{-\frac{1}{2}} &0 \\
%      0 & (\tilde{D}\!\!+\!\!I)^{-\frac{1}{2}}
%  \end{pmatrix}\!\!\!
% \begin{pmatrix}
%         \tilde{A} & I \\
%         I & \tilde{A}
%     \end{pmatrix}\!\!\!
% \begin{pmatrix}\!\!
%      (\tilde{D}\!\!+\!\!I)\!\!^{-\frac{1}{2}} &0 \\
%      0 & (\tilde{D}\!\!+\!\!I)\!\!^{-\frac{1}{2}}
% \!\! \end{pmatrix}\!\!\!
%  \begin{pmatrix}
%      H_n^{(l-1)}W^{(l)} \\
%      H_{ns}^{(l-1)}W^{(l)}
%  \end{pmatrix}\!\!\! \right)\\
%& =\!\!\sigma\!\!\left(\!\!\! \begin{pmatrix}
%  (\tilde{D}\!\!+\!\!I)^{-\frac{1}{2}}\!\!\tilde{A} & (\tilde{D}\!\!+\!\!I)^{-\frac{1}{2}}\\
%  (\tilde{D}\!\!+\!\!I)^{-\frac{1}{2}}&(\tilde{D}\!\!+\!\!I)^{-\frac{1}{2}}\!\!\tilde{A}
%  \end{pmatrix} \!\!\!
%\begin{pmatrix}
%     (\tilde{D}\!\!+\!\!I)^{-\frac{1}{2}} &0 \\
%     0 & (\tilde{D}\!\!+\!\!I)^{-\frac{1}{2}}
% \end{pmatrix}\!\!\!
% \begin{pmatrix}
%     H_{n}^{(l)}W^{(l)} \\
%     H_{ns}^{(l)}W^{(l)}
% \end{pmatrix} \!\!\!\right)\\
& = \sigma\left(\begin{pmatrix}
  (\tilde{D}+I)^{-\frac{1}{2}}\tilde{A}(\tilde{D} + I)^{-\frac{1}{2}} & \tilde{D} + I\\
  \tilde{D} + I & (\tilde{D} + I)^{-\frac{1}{2}}\tilde{A}(\tilde{D} + I)^{-\frac{1}{2}}
  \end{pmatrix} 
  \begin{pmatrix}
     H_{n}^{(l-1)}W^{(l)} \\
     H_{ns}^{(l-1)}W^{(l)}
 \end{pmatrix}  \right)\\
& = \sigma
% \!\!\left(
\begin{pmatrix}
  (\tilde{D} + I)^{-\frac{1}{2}}\tilde{A}(\tilde{D}+I)^{-\frac{1}{2}}H_{n}^{(l-1)}W^{(l)}+(\tilde{D}+I)H_{ns}^{(l-1)}W^{(l)}\\
  (\tilde{D}+I)^{-\frac{1}{2}}\tilde{A}(\tilde{D}+I)^{-\frac{1}{2}}H_{ns}^{(l-1)}W^{(l)}+(\tilde{D}+I)H_{n}^{(l-1)}W^{(l)}
  \end{pmatrix},
\end{aligned}
}
\end{equation}
where $\tilde{A}=A+I$, $\tilde{D}=D+I$ and $\tilde{D}^{'}=D^{'}+I$. $H_{n}^{(l-1)}$ and $H_{ns}^{(l-1)}$ is the node representation matrices of $G$ and $\hat{G}$ after the $(l-1)$-th convolutional layer. The feature input of the $0_{th}$ layer is node feature matrices $X^n$ and $X^{ns}$, which stack $\mathbf{x}_i^n$ and $\mathbf{x}_i^{ns}$ ($v_i\in G$) across all graph nodes, respectively. $H^{(l)}$ is the node representation matrix of all nodes after the $l_{th}$ convolutional layer. $W^{(l)}$ is the parameter matrix of the $l_{th}$ convolutional layer. $\sigma$ is a non-linear activation function. 

Since the original node features and augmented node structural features can be very different, we employ two different convolutional filters (i.e., with different convolutional weights) to learn their knowledge as follows:
\begin{equation}
    \begin{aligned}
    H^{(l)}=  \begin{pmatrix}
H_n^{(l)}\\
H_{ns}^{(l)}
 \end{pmatrix} 
\approx \sigma
% \!\!\left(
\begin{pmatrix}
  (\tilde{D}+I)^{-\frac{1}{2}}\tilde{A}(\tilde{D}+I)^{-\frac{1}{2}}H_{n}^{(l-1)}W^{(l)}_{n}+(\tilde{D}+I)H_{ns}^{(l-1)}W^{(l)}_{ns}\\
  (\tilde{D}+I)^{-\frac{1}{2}}\tilde{A}(\tilde{D}+I)^{-\frac{1}{2}}H_{ns}^{(l-1)}W^{(l)}_{ns}+(\tilde{D}+I)H_{n}^{(l-1)}W^{(l)}_{n}
  \end{pmatrix},
  % \!\!\!\right)
    \end{aligned}
\end{equation}
where $W^{(l)}_{n}$ and $W^{(l)}_{ns}$ are the parameter matrices of $l_{th}$ layer for two types of features respectively, and $H_n^{(l)}$ and $H_{ns}^{(l)}$ are the node representation matrices of $G$ and $\hat{G}$ after current $l_{th}$ message passing layer.

After $L$ message-passing layers, we aggregate the node representations of two graphs $G$ and $\hat{G}$ in each layer to obtain the final node representation matrix as follows:
\begin{equation}\label{eq:noderepre}
    \begin{aligned}
        H = \mathrm{AGGATE}_n(H_n^{(1)},\cdots, H_n^{(L)}, H_{ns}^{(1)},\cdots, H_{ns}^{(L)}), 
    \end{aligned}
\end{equation}
where $\mathrm{AGGATE}_n(\cdot)$ is an aggregate function, and concatenation is used in our experiments; 
$H$ denotes the representation matrix that encapsulates the representation of all individual nodes. Then a readout function is applied to obtain the learned graph representation $\mathbf{h}_l$.
%, $\mathbf{h}_l \in H$.
% $H^{(l)}$ represents the matrix composed of $\mathbf{h}_{v}^{(l)},v\in G$. 

\subsubsection{Message Passing in CoS-GIN} The framework can also be extended to other GNN backbones.
Here we now present how the proposed message passing method can be adopted to the case using GIN as our backbone. 
%GIN updates representation of node $v_i$ in $l_{th}$ layer through following:
%\begin{equation}
%    \begin{aligned}
%        \mathbf{h}^{(l)}_{v_i} = \mathrm{MLP}^{(l)}\left((1+\epsilon^{(l)})\mathbf{h}^{(l-1)}_{v_i}+\sum_{v_j\in\mathcal{N}(v_i)}\mathbf{h}^{(l-1)}_{v_j}\right),
%    \end{aligned}
%\end{equation}
%where $\mathrm{MLP}^{(l)}(\cdot)$ is the MLP in $l_{th}$ layer and $\epsilon^{(l)}$ is a parameter to be learned. And then the graph representation $\mathbf{h}_g$ for graph $G$ can be obtained by
%\begin{equation}
%    \begin{aligned}
%      \mathbf{h}_g\!\!=\!\! \mathrm{CONCAT}\!\left(\!\mathrm{READOUT}(\{\mathbf{h}_v^{(l)}|v\in G)\})|l\!\!=\!\!0,\!\!\cdots\!\!,L\!\! \right).
%    \end{aligned}
%\end{equation}
To this end, the GIN-based message passing is re-defined as follows:
\begin{equation}
    \begin{aligned}
        &\mathbf{h}^{(l)}_{v_i,n} = \mathrm{MLP}_n^{(l)}\left((1+\epsilon^{(l)})\mathbf{h}^{(l-1)}_{v_i,n}+\sum_{v_j\in\mathcal{N}(v_i)}\mathbf{h}^{(l-1)}_{v_j,n}+\mathbf{h}^{(l-1)}_{v_i,ns}\right),\\
       & \mathbf{h}^{(l)}_{v_i,ns}= \mathrm{MLP}_{ns}^{(l)}\left((1+\epsilon^{(l)})\mathbf{h}^{(l-1)}_{v_i,ns}+\sum_{v_j\in\mathcal{N}(v_i)}\mathbf{h}^{(l-1)}_{v_j,ns}+\mathbf{h}^{(l-1)}_{v_i,n}\right),
    \end{aligned}
\end{equation}
where MLP is a multi-layer perceptron layer. Then 
% instead of concatenating the node representations,
we combine the obtained representations via summation.
% after the fully-connected (FC) layer. 
In detail,
\begin{equation}
    \begin{aligned}
        &\mathbf{h}_{n} = \sum_{l}\mathrm{FC}^{(l)}_n(\mathrm{READOUT}(H^{(l)}_{n})),\\
        &\mathbf{h}_{ns} = \sum_{l}\mathrm{FC}^{(l)}_{ns}(\mathrm{READOUT}(H^{(l)}_{ns})),
    \end{aligned}
\end{equation}
where $\mathrm{FC}^{(l)}_n(\cdot)$ and $\mathrm{FC}^{(l)}_{ns}(\cdot)$ are fully-connected layers in the $l_{th}$ layer. We gain the learned graph representation $\mathbf{h}_l$ through adding them together:
\begin{equation}
    \begin{aligned}
        \mathbf{h}_l = \mathbf{h}_{n}+\mathbf{h}_{ns}.
    \end{aligned}
\end{equation}

The key insight of the message passing mechanism in CoS-GIN is analogous to that in CoS-GCN, but they are derived at different representation levels: matrix of node representations in CoS-GCN vs. vectorized node representations in CoS-GIN, which is mainly done for presentation brevity.

\subsection{Graph-level Representation Fusion}
After gaining the learned graph representations
% , a readout function is applied to obtain the learned graph representation
$\mathbf{h}_l$, we then employ MLP to synthesize it, together with the augmented graph-structural feature $\mathbf{x}^{gs}$, to learn the final graph representations. In detail, we input $\mathbf{h}_l$ and $\mathbf{x}^{gs}$ into the two different MLPs as:
\begin{equation}
    \begin{aligned}
        \mathbf{h}_l^{MLP} = \mathrm{MLP}^{l}(\mathbf{h}_l), \mathbf{h}_{gs}^{MLP} = \mathrm{MLP}^{gs}(\mathbf{x}^{gs}),
    \end{aligned}
\end{equation}
where $\mathrm{MLP}^{l}(\cdot)$ and $\mathrm{MLP}^{gs}(\cdot)$ are MLP functions. We then integrate the information learned to gain the final graph representation:
\begin{equation}\label{eq:graphrepresen}
    \begin{aligned}
    \mathbf{h}_g = \mathrm{AGGATE}_g(\mathbf{h}_l^{MLP}, \mathbf{h}_{gs}^{MLP}),    \end{aligned}
\end{equation}
where $\mathrm{AGGATE}_g(\cdot)$ is the aggregation function and we use concatenation in our model. Then the graph representation can be used for any down-stream tasks. Algorithm~\ref{cosgcn} presents the procedure of CoS-GCN to calculate graph representations, which can be later input to any down-stream tasks.

\renewcommand{\algorithmicrequire}{\textbf{Input:}}
\renewcommand{\algorithmicensure}{\textbf{Output:}}
\begin{algorithm}[h]
     \caption{Graph representation learning via CoS-GCN}\label{cosgcn} 
     \begin{algorithmic}[1]
     \REQUIRE Graph set $\mathcal{G}=\{G_i\}_i$, two GNNs with parameter set $\{W_n^{(1)},...,W_n^{(L)}\}$ and $\{W_{ns}^{(1)},...,W_{ns}^{(L)}\}$, two MLP functions $MLP^l(\cdot)$ and $MLP^{gs}(\cdot)$
     \ENSURE Graph representation $\mathbf{h}_g$ for $G \in \mathcal{G}$
     \STATE Augment node and graph structural knowledge to obtain $X^{ns}$ and $\mathbf{x}^{gs}$ for each $G\in\mathcal{G}$
     %\STATE Initialize $\{W_n^{(1)},...W_n^{(L)}\}$ and $\{W_{ns}^{(1)},...W_{ns}^{(L)}\}$
     \FOR {$G$ in $\mathcal{G}$}
     \STATE Compute $H^{(l)}, l\in\{1,\cdots,L\}$ with Eq.(4) 
     \STATE Aggregate $H^{(l)}, l\in\{1,\cdots,L\}$ with Eq.(5) to obtain $H$
     \STATE Readout $H$ to obtain $\mathbf{h}_l$
     \STATE Input $\mathbf{h}_l$ and $\mathbf{x}^{gs}$ into $MLP^l$ and $MLP^{gs}$ respectively to gain $\mathbf{h}_l^{MLP}$ and $\mathbf{h}_{gs}^{MLP}$ 
     \STATE Aggregate $\mathbf{h}_l^{MLP}$ and $\mathbf{h}_{gs}^{MLP}$ to obtain the final representation $\mathbf{h}_g$ for $G$
    \ENDFOR
   
    \RETURN Graph representation $\mathbf{h}_g$ for $G \in \mathcal{G}$
     \end{algorithmic}
 \end{algorithm}

%\subsection{Architecture}
%In this section, we try to display the details of how to train SHA-GNN with specific GNNs. 

%\subsubsection{SHA-GCN}
%To realize the information exchange of original and augmented features in SHA-GNN, for a graph $G$, we construct a new graph $\hat{G}$ with the same node and structure with the original graph but with the $\mathbf{x}^{ns}$ as its node attribute and link their corresponding node to obtain the new graph $G^{'}$. 

\subsection{Expressive Power of CoS-GNN}
This section discusses the expressive power of CoS-GNN. When comparing the expressiveness of GNN models, we can define that:

\begin{myDef}
For any two GNN models: A and B, model A is said to be more expressive than model B, if and only if 1) model A can distinguish all samples that model B can distinguish, and 2) there exists samples which can be distinguished by model A but not by model B.     
\end{myDef}

To measure the expressive power of GNNs, the Weisfeiler-Lehman (WL) graph isomorphism test is commonly used, which is a family of algorithms (k-WL, k-FWL) used to test graph isomorphism~\cite{maron2019provably, grohe2017descriptive}. Two graphs $G_1$ and $G_2$ are called isomorphic if there exists an edge and color preserving bijection $\phi:\mathcal{V}_1\rightarrow\mathcal{V}_2$. Next we show the strong expressive power of our model CoS-GNN from the WL-test perspective:
\begin{theorem}
CoS-GNN is not less expressive than 1-WL and 2-WL tests.
% as least.   
\end{theorem}

\begin{proof}
We first consider the comparison with 1-WL test. This equals to prove such statement: If CoS-GNN deems that two graphs are isomorphic, then 1-WL test will also deem them isomorphic. If after k iterations, the CoS-GNN regards two graphs $G_1$ and $G_2$ are isomorphic, we have $\mathbf{h}_{1,g}^{(k)}=\mathbf{h}_{2,g}^{(k)}$. Assuming that the $\mathrm{AGGATE_g}$ is injective, we can obtain that $\mathbf{h}_{1,l}^{MLP(k)}=\mathbf{h}_{2,l}^{MLP(k)}$ and $\mathbf{h}_{1,gs}^{MLP(k)}=\mathbf{h}_{2,gs}^{MLP(k)}$, followed by $\mathbf{h}_{1,l}^{(k)}=\mathbf{h}_{2,l}^{(k)}$ and $\mathbf{x}^{gs}_1=\mathbf{x}^{gs}_2$. Thus we have $H_1^{(i)}=H_2^{(i)}$ and then $\mathbf{h}_{v,n}^{(i)}=\mathbf{h}_{u,n}^{(i)}$ and $\mathbf{h}_{v,ns}^{(i)}=\mathbf{h}_{u,ns}^{(i)}$ for $v\in\mathcal{V}_{G_1}$, $u\in\mathcal{V}_{G_2}$ and $i=1,...,k$ when the $\mathrm{AGGATE_n}$ is injective. 

What we need to prove next is that the color extracted by 1-WL for node $v$ and $u$ is same, \ie $c_v^{(k)}=c_u^{(k)}$. We use the induction as \cite{bouritsas2022improving} to demonstrate this. For $i=0$, since the initial node features are the same for both CoS-GNN and 1-WL, we can get $c_v^{(0)}=c_u^{(0)}$ when $\mathbf{h}_{v,n}^{(0)}=\mathbf{h}_{u,n}^{(0)}$. Suppose $\mathbf{h}_{v,n}^{(j)}=\mathbf{h}_{u,n}^{(j)}, \mathbf{h}_{v,ns}^{(j)}=\mathbf{h}_{u,ns}^{(j)}\Rightarrow c_v^{(j)}=c_u^{(j)}$ holds for $j=1,\cdots,k-1$, we later need to prove that it holds for $j=k$. Since each node representation, including $\mathbf{h}_{v,n}^{(j)}$ and $\mathbf{h}_{v,ns}^{(j)}$, is calculated by a $\mathrm{COM}$ function, if $\mathrm{COM}$ is injective, we have $\mathbf{h}_{v,n}^{(k-1)}=\mathbf{h}_{u,n}^{(k-1)}$, $\mathbf{h}_{v,ns}^{(k-1)}=\mathbf{h}_{u,ns}^{(k-1)}$, $\mathrm{AGGATE}(\{\mathbf{h}_{q,n}^{(k-1)}|q\in\mathcal{N}_v\})=\mathrm{AGGATE}(\{\mathbf{h}_{p,n}^{(k-1)}|p\in\mathcal{N}_u\})$ and $\mathrm{AGGATE}(\{\mathbf{h}_{q,ns}^{(k-1)}|q\in\mathcal{N}_v\})=\mathrm{AGGATE}(\{\mathbf{h}_{p,ns}^{(k-1)}|p\in\mathcal{N}_u\})$ when $\mathbf{h}_{v,n}^{(k)}=\mathbf{h}_{u,n}^{(k)}$ and $\mathbf{h}_{v,ns}^{(k)}=\mathbf{h}_{u,ns}^{(k)}$. According to Lemma 5 from \cite{xupowerful}, there exists an injective function. When $\mathrm{AGGATE}$ is injective, we have $\mathbf{h}_{q,n}^{(k-1)}=\mathbf{h}_{p,n}^{(k-1)}$ and $\mathbf{h}_{q,ns}^{(k-1)}=\mathbf{h}_{p,ns}^{(k-1)}$, which lead to $c_q^{(k-1)}=c_p^{(k-1)}$ for 
$q\in\mathcal{N}_v$ and $p\in\mathcal{N}_u$. Since we have $c_u^{(k-1)}=c_v^{(k-1)}$ according to the induction hypothesis, we can get $c_u^{(k)}=c_v^{(k)}$. Therefore, the 1-WL test regards two graphs isomorphic if the CoS-GNN regards them isomorphic. 

Since 1-WL and 2-WL test have equivalent discrimination power~\cite{maron2019provably,zhang2021nested}, CoS-GNN is also at least as expressive as 2-WL test.
\end{proof}
The theorem states that CoS-GNN is at least as expressive as 1-WL and 2-WL tests. % Detailed proof is presented in Appendix. 
Some graphs that 1-WL and 2-WL tests cannot distinguish can be identified by our CoS-GNN. For example, 1-WL and 2-WL fail to distinguish the two graphs in Figure~\ref{fig:WLfail}, whereas CoS-GNN can easily differentiate them with the augmented features. Thus, our CoS-GNN can often learn more expressive representations than popular GNNs since they are mainly based on the 1-WL test, when handling complex graph datasets. For example, \citet{chen2020can} have shown that MPNNs cannot perform induced-subgraph-count of any connected pattern consisting of 3 or more nodes. For graphs with subgraphs that MPNNs cannot learn to count, there would be some pairs of graphs with different number of such uncounted subgraphs that are regarded as isomorphic by MPNNs. On the other hand, CoS-GNN can discriminate these graphs through including structural features that differentiate these subgraphs. As shown in Figure~\ref{fig:WLfail}, the two graphs cannot be distinguished by MPNNs, but they can be differentiated by the triangle counting for both nodes and graphs, and the existence of bridge in the graphs as well.

When compared with higher-order WL tests, we can also observe that our CoS-GNN can distinguish graphs that 2-FWL test (which is equivalent to 3-WL test~\cite{maron2019provably}) fails to identify, meaning that 3-WL test is not more expressive than our CoS-GNN. For example, \citet{arvind2020weisfeiler} and \citet{bouritsas2022improving} have shown that the 2-FWL test fails to distinguish the well-known Rook’s 4×4 and Shrikhande graphs, as illustrated in Figure~\ref{rookshrik}. However, the clique features incorporated into our CoS-GNN model help effectively discriminate these two graphs.

\begin{figure}
    \centering
    \includegraphics[width=8cm]{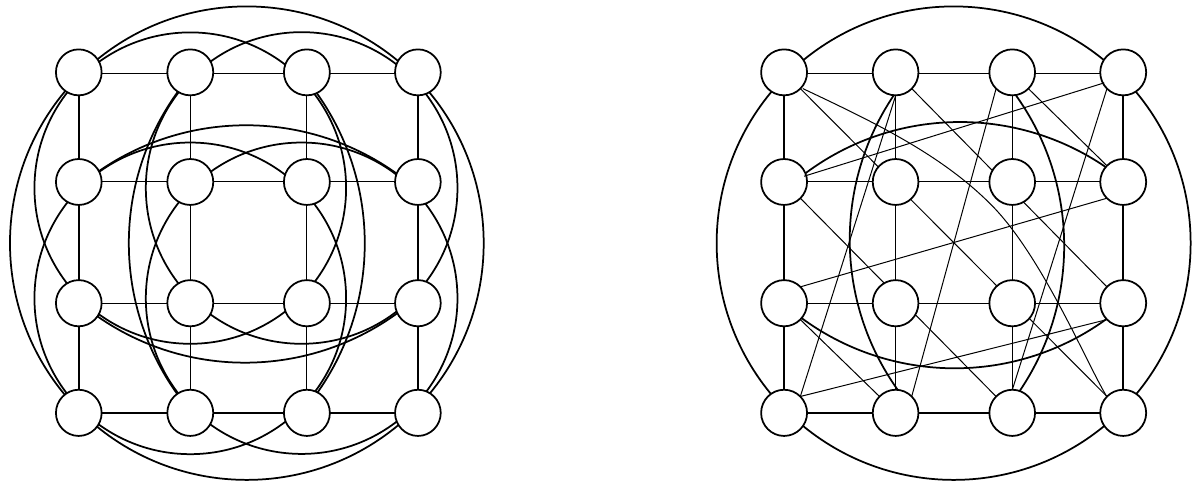}
    \caption{The strongly regular Rook’s 4×4 graph (left) and Shrikhande graph (right) \cite{bouritsas2022improving, arvind2020weisfeiler}. The 3-WL/2-FWL test is not able to deem them as non-isomorphic. Rook’s 4×4 graph possesses 4-cliques while the Shrikhande graph features 5-rings, which are not present in Rook’s.}
    \label{rookshrik}
\end{figure}

\subsection{Time Complexity Analysis}
In this section, we analyze the time complexity of CoS-GNN. The computation cost mostly concentrates on the feature extraction stage and the message passing stage. Let $n$ and $m$ be the number of nodes and edges in the graph respectively, in the feature learning phase, the degree and triangle counting cost are $\mathcal{O}(n)$ and $\mathcal{O}(n^2)$ time respectively. The complexity of clique and core finding are respectively bounded by $\mathcal{O}(n*3^n)$ and $\mathcal{O}(n+m)$. The computation of triangle and square clustering coefficient is $\mathcal{O}(n^2)$. The bridge finding needs $\mathcal{O}(n+m)$ time. The average clustering coefficient, average global and local efficiency require $\mathcal{O}(n^2)$, $\mathcal{O}(n^3)$ and $\mathcal{O}(n^4)$ respectively. Therefore, the feature extraction stage requires $\mathcal{O}(n^4+n*3^n+m)$ time. As for the message passing stage, the time complexity of our CoS-GNN equals to the corresponding vanilla GNN. Thus, the total time complexity of our CoS-GNN is $\mathcal{O}(n^4+n*3^n+m)+\mathcal{O}_{GNN}$.

\section{Experiments and Results}
\subsection{Datasets}
We perform experiments on 12 publicly available datasets from the TUDataset graph classification benchmark~\cite{Morris+2020}
to justify the effectiveness of our CoS-GNN. The detailed information of the datasets is displayed in Table~\ref{datasets}.

\begin{table}[htbp]
\caption{The detailed information of 12 public datasets. The following acronyms, PROTEINS$\_$full (PROTS$\_$full), IMDB-BINARY (I-BINARY), IMDB-MULTI (I-MULTI), REDDIT-BINARY (R-BINARY) and
REDDIT-MULTI-5K (R-MULTI), are used. The `binary' in the `Class' column denotes the dataset is for binary classification while `multi' implies multi-class classification. The `\#Graphs' is the total number of graphs in the dataset and the `\#Nodes' means the average number of nodes in the dataset. The `$\checkmark$' in the `Attribute' column indicates the data contains attributed graphs, and otherwise they contain only plain graphs.}
\centering
\label{datasets}
\setlength{\tabcolsep}{0.6mm}
\scalebox{0.8}{
\begin{tabular}{l|ccccc}
\hline\hline
\textbf{Dataset} & \textbf{Area} & \textbf{Class}&  \textbf{\#Graphs} &  \textbf{\#Nodes} & \textbf{Attribute} \\
\hline
BZR & molecule & binary &  405 & 35.75 & $\checkmark$ \\
COX2 & molecule & binary &  467 &  41.22 & $\checkmark$\\
%DHFR & molecule & binary & $\checkmark$ \\
DD & bioinformatics & binary &  1178 &  284.32 & - \\
I$-$BINARY & social & binary & 1000 & 19.77 & - \\
I$-$MULTI & social & multi & 1500  &  13.00 &- \\
MUTAG & molecule & binary & 188  & 17.93 &- \\
NCI1 & molecule & binary & 4110  &  29.87&-  \\
NCI109 & molecule & binary & 4127& 29.68 &-   \\
%PROTEINS & bioinformatics & binary & $\checkmark$ \\
PROTS$\_$full & bioinformatics & binary &  1113  & 39.06 &$\checkmark$	 \\
R$-$BINARY & social & binary & 	2000 & 	429.63& - \\
R$-$MULTI & social & multi & 	4999 &	508.52 &- \\
ENZYMES & bioinformatics & multi &600 &  32.63 &$\checkmark$ \\
\hline\hline
\end{tabular}
}
\end{table}

\subsection{Competing Methods and Evaluation Metrics}
Our method CoS-GNN is compared with 13 state-of-the-art (SOTA) methods:
\begin{itemize}
    \item \textbf{Graph kernels.} We use two graph kernels, \ie \textbf{Weisfeiler-lehman subtree kernel (WL)}~\cite{shervashidze2011weisfeiler} and \textbf{Propagation graph kernels (PK)}~\cite{neumann2016propagation} as baselines.
    \item \textbf{Basic graph neural networks.} We consider four popular networks, \ie, \textbf{GCN}~\cite{kipfsemi}, \textbf{SAGE}~\cite{hamilton2017inductive}, \textbf{GAT}~\cite{velivckovicgraph} and \textbf{GIN}~\cite{xupowerful}, as the network baselines.
    \item \textbf{GNN-based augmentation methods.} We also compare CoS-GNN with several augmentation models that are built based on GNNs, including $\mathbf{\mathcal{G}}$\textbf{-mixup}~\cite{han2022g}, \textbf{Dummy} \cite{liu2022boosting}, \textbf{DropGNN}~\cite{papp2021dropgnn}, \textbf{rGIN}~\cite{sato2021random}, \textbf{NestedGNN}~\cite{zhang2021nested}, \textbf{LAGNN}~\cite{liu2022local}, and \textbf{GSN}~\cite{bouritsas2022improving}. 
\end{itemize}

In terms of performance evaluation, we employ accuracy and Area Under Precision-Recall Curve (AUPRC) as the evaluation metrics for graph classification while Area Under Receiver Operating Characteristic Curve (AUC) for anomaly detection. Higher accuracy/AUPRC/AUC indicates better performance. We report the mean results and standard deviation based on 10-fold cross-validation for all datasets. 

\subsection{Implementation Details}
All experiments are executed on NVIDIA Quadro RTX 6000 GPU
with an Intel Xeon E-2288G 3.7GHz CPU, and all models are implemented with Python 3.8\footnote{https://www.python.org/}. 
The following parameters are set by default for CoS-GCN and its competing methods, including WL, PK, GAT, SAGE and GCN, on all 12 datasets: the learning rate is 0.001, the batch size is set to 512, the number of network layers is 3, the hidden layer dimension of network is 256, the classifier is a 3-layer MLP, pooling operation is max pooling, and the number of epochs is 1,000. The iteration number of WL is 3. For GIN and CoS-GIN, the learning rate is chosen from $\{0.01, 0.001, 0.0005, 0.0001\}$, the batch size is selected from $\{32, 64, 128, 256\}$, hidden layer dimension is ranged in $\{16, 64, 128, 256\}$ and the readout operation is either meanpooling or maxpooling. For other baselines, we run their public codes with their recommended settings. 

\begin{landscape}
\begin{table}[t]
\caption{Accuracy (mean$\pm$std) of CoS-GNN and SOTA competing methods for graph classification on 12 real-world datasets. The best and second performance per dataset is boldfaced and underlined respectively. 
The following acronyms, PROTEINS$\_$full (PROTS$\_$full), IMDB-BINARY (I-BINARY), IMDB-MULTI (I-MULTI), REDDIT-BINARY (R-BINARY) and REDDIT-MULTI-5K (R-MULTI), are used. `Rank' indicates the average performance ranking of a model across all datasets: a smaller rank value indicates a better overall performance.  }
\label{Accuracy}
\setlength{\tabcolsep}{0.3mm}
\scalebox{0.7}{
\begin{tabular}{lcccccccccccccc}
\hline\hline
\textbf{Method} & \textbf{BZR} & \textbf{COX2} &\textbf{DD}  & \textbf{I-BINARY} & \textbf{I-MULTI} & \textbf{MUTAG} & \textbf{NCI1} &  \textbf{NCI109}  & \textbf{PROTS$\_$full}& \textbf{R-BINARY}& \textbf{R-MULTI}& \textbf{ENZYMES}&\textbf{Rank}&\textbf{p-value}\\
\hline
\textbf{WL} & 0.790$\pm$0.045 & 0.760$\pm$0.043 &0.747$\pm$0.023 & 0.727$\pm$0.035 &0.503$\pm$0.021& 0.797$\pm$0.072 &0.643$\pm$0.025 &0.647$\pm$0.021  & 0.737$\pm$0.028	& 0.611$\pm$0.039	&0.382$\pm$0.019 & 0.338$\pm$0.074 & 13.5& 0.0005\\
\textbf{PK} & 0.818$\pm$0.026& 0.790$\pm$0.015 &0.742$\pm$0.026&   0.729$\pm$0.050& 0.488$\pm$0.038 & 0.697$\pm$0.054&0.634$\pm$0.017&0.620$\pm$0.036 &  0.734$\pm$0.026& 0.635$\pm$0.049& 0.411$\pm$0.024& 0.392$\pm$0.036&   14.3&0.0005 \\
\hline
\textbf{GCN} & 0.840$\pm$0.028	&0.811$\pm$0.043 & 0.756$\pm$0.033		& 0.724$\pm$0.036 &0.498$\pm$0.024 & 0.744$\pm$0.102 &0.785$\pm$0.016 & 0.769$\pm$0.026 &   0.749$\pm$0.028	&  0.900$\pm$0.024	&  0.533$\pm$0.013	& 0.450$\pm$0.068	&  8.8 &0.0005\\
\textbf{GAT} & 0.837$\pm$0.036	& 0.790$\pm$0.065 &0.773$\pm$0.027	& 0.704$\pm$0.044 &0.496$\pm$0.035 &0.738$\pm$0.105 & 0.772$\pm$0.015 & 0.765$\pm$0.027 &  0.748$\pm$0.035	& 0.851$\pm$0.036	&  0.502$\pm$0.025	&0.422$\pm$0.069	& 10.8&0.0010 \\
\textbf{SAGE} & 0.842$\pm$0.055	& 0.820$\pm$0.055 &  0.771$\pm$0.031	 &  0.733$\pm$0.046 & 0.495$\pm$0.024 &0.766$\pm$0.059 & 0.795$\pm$0.018 & 0.780$\pm$0.027  &0.748$\pm$0.034	& 0.878$\pm$0.029	& 0.508$\pm$0.017	&  0.395$\pm$0.079	&  8.3&0.0005 \\
\textbf{GIN} & 0.835$\pm$0.033	&0.805$\pm$0.039 & 0.750$\pm$0.048	 & 0.731$\pm$0.047 & 0.502$\pm$0.031 &0.866$\pm$0.078 &0.790$\pm$0.023 &0.795$\pm$0.012	 &0.738$\pm$0.039	& 0.891$\pm$0.016	& \underline{0.557}$\pm$0.021	& 0.550$\pm$0.091	& 7.3&  0.0005\\
\hline
\textbf{$\mathcal{G}-$mixup} &0.832$\pm$0.042	&0.805$\pm$0.043 &$-$ 	 & 0.719$\pm$0.030 & 0.505$\pm$0.015  &0.824$\pm$0.093 &0.778$\pm$0.023  & 0.752$\pm$0.039& 0.705$\pm$0.054 & \textbf{0.929}$\pm$0.009	&0.555$\pm$0.005 &   0.487$\pm$0.054	&8.9 & 0.0049\\
\textbf{Dummy} & 0.828$\pm$0.069	 	& 0.805$\pm$0.048	      & 0.777$\pm$0.035	 	&  $-$  &$-$ & 0.829$\pm$0.105 & 0.752$\pm$0.015 & 0.738$\pm$0.025&  \textbf{0.770}$\pm$0.029	& $-$	&  $-$	&  0.500$\pm$0.101	&  8.4& 0.0034 \\
\textbf{DropGNN} &\underline{0.845}$\pm$0.030 & 0.807$\pm$0.055 & $-$  &  0.741$\pm$0.027&0.502$\pm$0.023& 0.834$\pm$0.095& 0.809$\pm$0.014 & \underline{0.810}$\pm$0.013  & 0.747$\pm$0.032 &  $-$& $-$ & 0.588$\pm$0.053 & 5.7&  0.0005\\
\textbf{rGIN} & 0.827$\pm$0.046 & 0.805$\pm$0.051& 0.698$\pm$0.020 &  0.746$\pm$0.022&\underline{0.512}$\pm$0.026&0.851$\pm$0.052&0.808$\pm$0.019&  0.801$\pm$0.016&  0.757$\pm$0.024& 0.897$\pm$0.020& 0.549$\pm$0.023&   \underline{0.663}$\pm$0.073 & 6.0&0.0098\\
\textbf{NestedGCN} & 0.808$\pm$0.033& 0.782$\pm$0.009&0.763$\pm$0.038& 0.733$\pm$0.052&0.499$\pm$0.037&0.829$\pm$0.011&0.720$\pm$0.080&0.708$\pm$0.076 &0.730$\pm$0.017&  $-$& $-$ &  0.312$\pm$0.067 & 12.8 &0.0005 \\
\textbf{NestedGIN} & 0.832$\pm$0.082& 0.790$\pm$0.066& \textbf{0.778}$\pm$0.039& 0.745$\pm$0.064& 0.510$\pm$0.023& \underline{0.879}$\pm$0.082& 0.777$\pm$0.017& 0.779$\pm$0.026& 0.738$\pm$0.035& $-$& $-$& 0.290$\pm$0.080& 7.8 & 0.0015\\
\textbf{GSN-v} & 0.825$\pm$0.036& 0.801$\pm$0.040 &0.702$\pm$0.033&   \textbf{0.760}$\pm$0.047&  0.509$\pm$0.031& 0.861$\pm$0.097& 0.808$\pm$0.018& 0.804$\pm$0.020&   0.736$\pm$0.036& 0.819$\pm$0.028&0.521$\pm$0.026& \textbf{0.687}$\pm$0.043 & 7.6& 0.0190\\
\textbf{LAGCN} &0.837$\pm$0.039& 0.807$\pm$0.064& 0.752$\pm$0.033&     0.727$\pm$0.041&	0.499$\pm$0.022&0.761$\pm$0.087& 0.752$\pm$0.027& 0.713$\pm$0.031& 0.755$\pm$0.035&0.884$\pm$0.016& 0.513$\pm$0.026&0.590$\pm$0.047 & 9.3 &0.0010\\
\textbf{LAGIN} & 0.820$\pm$0.035 &  \textbf{0.831}$\pm$0.028&   0.766$\pm$0.037& 0.731$\pm$0.055& 0.491$\pm$0.032& 0.872$\pm$0.049&0.763$\pm$0.016 & 0.740$\pm$0.032& \underline{0.760}$\pm$0.032& 	0.894$\pm$0.020	&  	0.553$\pm$0.022 & 0.645$\pm$0.042  & 7.6	&0.0156\\ 
\hline
\textbf{CoS-GCN} & 0.832$\pm$0.036	&\underline{0.824}$\pm$0.055&\textbf{0.778}$\pm$0.032	 &\underline{0.754}$\pm$0.052 & 0.503$\pm$0.019 &  0.878$\pm$0.059 &\underline{0.816}$\pm$0.013 & 0.801$\pm$0.021 & 0.757$\pm$0.028	&  \underline{0.917}$\pm$0.022	& 0.554$\pm$0.028 & 0.533$\pm$0.064	& \underline{3.8}& $-$ \\
\textbf{CoS-GIN} & \textbf{0.850}$\pm$0.043 & 0.822$\pm$0.094&   0.773$\pm$0.038 & 0.753$\pm$0.033 & \textbf{0.517}$\pm$0.023& \textbf{0.883}$\pm$0.066&\textbf{0.823}$\pm$0.026 & \textbf{0.812}$\pm$0.013&  0.751$\pm$0.019 &  0.906$\pm$0.018& \textbf{0.559}$\pm$0.021& 0.653$\pm$0.036&\textbf{2.3} & $-$	\\   
\hline\hline
\end{tabular}
}
\end{table}
\end{landscape}

\begin{landscape}
\begin{table}[t]
\centering
\caption{AUPRC (mean$\pm$std) of CoS-GNN and SOTA competing methods for graph classification on 9 real-world binary classification datasets. The best and second performance per dataset is boldfaced and underlined respectively. 
The following acronyms, PROTEINS$\_$full (PROTS$\_$full), IMDB-BINARY (I-BINARY), IMDB-MULTI (I-MULTI), REDDIT-BINARY (R-BINARY) and REDDIT-MULTI-5K (R-MULTI), are used. `Rank' indicates the average performance ranking of a model across all datasets: a smaller rank value indicates a better overall performance.}
\label{AUPRC}
\setlength{\tabcolsep}{0.3mm}
\scalebox{0.7}{
\begin{tabular}{lccccccccccc}
\hline\hline
\textbf{Method} & \textbf{BZR} & \textbf{COX2} &\textbf{DD}  & \textbf{I-BINARY} & \textbf{MUTAG} & \textbf{NCI1} &  \textbf{NCI109}  & \textbf{PROTS$\_$full}& \textbf{R-BINARY}& \textbf{Rank}\\
\hline
\textbf{GCN}& 0.674$\pm$0.124&0.541$\pm$0.152&0.760$\pm$0.055&0.831$\pm$0.038&	0.783$\pm$0.113&0.842$\pm$0.023&0.824$\pm$0.0198&\underline{0.763}$\pm$0.038&0.956$\pm$0.011& 7.3\\
\textbf{GAT}&0.624$\pm$0.156&0.520$\pm$0.137&0.753$\pm$0.058&0.814$\pm$0.042&	0.870$\pm$0.089&0.841$\pm$0.020&0.811$\pm$0.014&\textbf{0.772}$\pm$0.038&0.937$\pm$0.022 & 9.6\\
\textbf{SAGE}&0.661$\pm$0.170&0.565$\pm$0.144&\textbf{0.803}$\pm$0.038&0.813$\pm$0.042&0.899$\pm$0.0569&0.852$\pm$0.029&0.823$\pm$0.0176&0.760$\pm$0.041&0.944$\pm$0.015 &6.8\\
\textbf{GIN}&0.619$\pm$0.123&0.556$\pm$0.115&0.770$\pm$0.037&0.828$\pm$0.032	&0.963$\pm$0.026&0.857$\pm$0.028&0.848$\pm$0.021&0.699$\pm$0.058&0.940$\pm$0.034&7.3\\
\hline
\textbf{$\mathcal{G}-$mixup}&0.691$\pm$0.130&0.439$\pm$0.160& $-$ &	0.829$\pm$0.033&0.945$\pm$0.040&0.833$\pm$0.021&0.796$\pm$0.043&0.688$\pm$0.078& $-$  &10.6\\	
\textbf{Dummy}&0.570$\pm$0.067&0.554$\pm$0.175&\underline{0.800}$\pm$0.040&$-$&	0.915$\pm$0.074&0.808$\pm$0.031&0.786$\pm$0.027&\textbf{0.772}$\pm$0.040&$-$	&9.6\\
\textbf{DropGNN}&0.675$\pm$0.085&0.555$\pm$0.172&0.690$\pm$0.044&\underline{0.837}$\pm$0.024&0.962$\pm$0.027&\textbf{0.879}$\pm$0.020&\underline{0.863}$\pm$0.021&0.708$\pm$0.051&0.938$\pm$0.029&6.1 \\
\textbf{rGIN}&0.576$\pm$0.096&\underline{0.573}$\pm$0.147&0.704$\pm$0.026&0.821$\pm$0.022	&0.968$\pm$0.025&0.868$\pm$0.020&0.862$\pm$0.010&0.706$\pm$0.068&0.928$\pm$0.039 & 7.7\\
\textbf{NestedGIN}&0.587$\pm$0.233&0.563$\pm$0.180&0.784$\pm$0.049&\textbf{0.840}$\pm$0.039	&0.962$\pm$0.029&0.851$\pm$0.020&0.842$\pm$0.025&0.745$\pm$0.050&$-$&	6.9\\
\textbf{GSN-v}&\underline{0.694}$\pm$0.143&0.528$\pm$0.150&0.699$\pm$0.050&0.828$\pm$0.073	&0.969$\pm$0.031	&0.857$\pm$0.0215&0.847$\pm$0.021&0.670$\pm$0.132&0.876$\pm$0.037& 8.2 \\
\textbf{LAGCN}&0.652$\pm$0.077&0.560$\pm$0.113&0.757$\pm$0.045	&0.826$\pm$0.038	&0.912$\pm$0.034	&0.800$\pm$0.020	&0.755$\pm$0.026	&0.747$\pm$0.040	&\underline{0.961}$\pm$0.025& 8.8 \\
\textbf{LAGIN}&0.552$\pm$0.089&	0.539$\pm$0.130&	0.785$\pm$0.066&	0.833$\pm$0.025	&	\textbf{0.979}$\pm$0.012&	0.838$\pm$0.024&	0.811$\pm$0.044&	0.723$\pm$0.067&0.955$\pm$0.020 & 7.7\\
\hline
\textbf{CoS-GCN}&	0.594$\pm$0.102	&\textbf{0.575}$\pm$0.133	&0.795$\pm$0.052&	\underline{0.837}$\pm$0.046	&	0.974$\pm$0.020&	\textbf{0.879}$\pm$0.014	&0.861$\pm$0.027	&0.760$\pm$0.039	&0.960$\pm$0.022& \textbf{3.4}\\
\textbf{CoS-GIN}&	\textbf{0.724}$\pm$0.147&	0.562$\pm$0.160&	0.777$\pm$0.046&	0.829$\pm$0.034	&	\underline{0.977}$\pm$0.019&	\underline{0.872}$\pm$0.015&	\textbf{0.880}$\pm$0.021&	0.725$\pm$0.035&	\textbf{0.963}$\pm$0.011& \underline{3.7}\\

\hline\hline
\end{tabular}
}
\end{table}
\end{landscape}

\subsection{Enabling Graph Classification}\label{subsec:sota}
% We measure the performance of CoS-GNN in graph classification task in this section. 
The graph classification accuracy results of CoS-GNN models (including CoS-GCN and CoS-GIN) and 12 SOTA competing methods are reported in Table~\ref{Accuracy}, where the GNN backbone used in $\mathcal{G}$-mixup, Dummy and DropGNN is all GIN due to its better performance; the results of $\mathcal{G}$-mixup on the IMDB and REDDIT datasets are taken from \cite{han2022g}; the result of Dummy on DD, NCI1 and NCI109 are from \cite{liu2022boosting}; the results of NestedGCN and NestedGIN on DD, MUTAG and ENZYMES are from \cite{zhang2021nested}; and `-' means the results are not reported in the original papers. 
% The last column is the average rank of each method. 

It is clear that CoS-GIN and CoS-GCN achieve the best or second-best performance on most of the datasets and the two top-ranked methods among all methods. Specifically, CoS-GCN improves GCN by $0.8\%$, $2.2\%$, $3.0\%$, $3.1\%$, $3.2\%$ and $8.3\%$ for PROTEINS$\_$full, DD, IMDB-BINARY, NCI1, NCI109 and ENZYMES respectively, while the improvements brought by CoS-GIN over GIN are $1.5\%$, $1.7\%$, $1.7\%$, $2.2\%$, $2.3\%$, $3.3\%$ and $10.3\%$ for BZR, COX2, IMDB-BINARY, DD, NCI1, NCI109 and ENZYMES respectively. These large performance advancement reveals that the structural information in these dataset is specific and the feature augmentation and message passing process in our CoS-GNN makes full use of these structural information to improve its performance. When compared with other augmentation methods, our models can also perform better than the SOTA models on most datasets (\ie, NCI109 ($0.2\%$), REDDIT-MULTI ($0.2\%$), MUTAG ($0.4\%$), BZR ($0.5\%$), IMDB-MULTI ($0.5\%$) and NCI1 ($1.4\%$)) and ranks top among all the competitors on overall performance. We also perform a paired Wilcoxon signed rank test to examine the significance of CoS-GNN against each of the competing methods across the 12 datasets.
As shown by the p-values in Table~\ref{Accuracy}, our CoS-GIN significantly outperforms GSN-v and LAGIN at the $95\%$ confidence level and exceeds other competitors at the $99\%$ confidence level. These results indicate that our collective node and graph structural knowledge augmented GNNs can learn more important graph structure information for graph classification.  Besides, on individual datasets, CoS-GNN can gain 2\%-11\% accuracy improvement maximally on specific datasets 
when compared to the best-performing competing methods NestedGNN, GSN-v and LAGNN (for example, 5\% enhancement of NestedGIN on NCI1, 9\% improvement of GSN on REDDIT-BINARY). 
%the more excellent performance of CoS-GNN than NestedGNN, GSN-v and LAGNN 
This means that the domain-adaptive graph structural knowledge in CoS-GNN can provide more generalized information to improve the model performance across different datasets while NestedGNN, GSN-v and LAGNN only consider the local structural information, which limits their performance. In summary, compared to each SOTA method, CoS-GNN may only have limited improvements on a few individual datasets, but the improvement on a set of datasets is substantial, and its improvement is significant across the 12 datasets used.

We report the AUPRC results of CoS-GNN and the competing methods on binary classification tasks in Table \ref{AUPRC}. Considering the limited performance of WL, PK and LAGCN, we omit their results. As can be seen in Table \ref{AUPRC}, although our CoS-GNN is not always the best model on every dataset, our CoS-GIN and CoS-GCN still achieve the top two performance on overall datasets, which further demonstrates the excellent ability of our CoS-GNN. 

We also compare our CoS-GNN with vanilla GNNs on Open Graph Benchmark (OGB) dataset--ogbg-molhiv and ogbg-molpcba in Table \ref{ogb}. Our CoS-GNN achieves better performance than corresponding vanilla GNN in most situations, indicating the positive contribution of the augmented features. The performance of CoS-GIN is a bit worse than that of GIN on ogbg-molpcba, which might be because that although our augmented features are useful, which is demonstrated by the improvement of CoS-GCN compared with GCN, GIN has also learned enough useful structural information and the augmentation operation in our CoS-GIN does not provide extra discriminative information.
%Our CoS-GIN achieves the best performance when compared with other competitors. Compared to GIN alone, CoS-GIN shows $0.9\%$ performance gains. Among the competitors, rGIN obtains highest AUC, which is $0.5\%$ lower than our CoS-GIN, indicating the effect brought by the extracted structural knowledge. 

\begin{table}[t]
\centering
\caption{Results (mean$\pm$std) of CoS-GNN and corresponding vanilla GNN on OGB datasets -- ogbg-molhiv and ogbg-molpcba. The best performance per dataset is boldfaced.} 
\label{ogb}
\setlength{\tabcolsep}{0.3mm}
\scalebox{0.7}{
\begin{tabular}{cccccccc}
\hline\hline
 & ogbg-molhiv & ogbg-molpcba \\
\textbf{Model} & \textbf{AUROC} & \textbf{AP} \\
\hline
\textbf{GCN} & 0.7626$\pm$0.0098   &  0.1753$\pm$0.0023  \\
\textbf{CoS-GCN} & $\mathbf{0.7662\pm0.0165}$  & $\mathbf{0.2045\pm0.0034}$ \\
\hline
\textbf{GIN} & 0.7825$\pm$0.0077   &   $\mathbf{0.2288\pm0.0027}$  \\
\textbf{CoS-GIN} & $\mathbf{0.7912\pm0.0068}$  &  0.2249$\pm$0.0034 \\
\hline\hline
\end{tabular}
}
\end{table}

We also calculate the training and inference time of our CoS-GNN and its competitors to demonstrate the efficiency of the CoS-GNN. We use the same GIN structure in all models. The results are reported in Table~\ref{testtime}. We can see that our CoS-GIN is a little more costly than simple augmentation operation with conventional GIN module, which is caused by the feature augmentation operation, but it is more efficient than complex structural augmentation methods, including DropGNN and NestedGIN. Besides, although our CoS-GIN is a little time-costly on some large-scale datasets, it can still be successfully implemented on devices with limited computational ability, while $\mathcal{G}$-mixup, dummy, DropGNN and NestedGIN require more powerful devices on such instances, which restricts their application.

%Although our CoS-GIN is more time-costly on those large-scale datasets, more powerful computational ability is not necessary for our CoS-GIN, while $\mathcal{G}$-mixup, dummy, DropGNN and NestedGIN require more powerful devices on such instances. }

%The substructure counting in GSN makes it cost more training time than our CoS-GIN. In addition, our CoS-GIN also can run on samples with large scale and dense structure, while $\mathcal{G}$-mixup, dummy, DropGNN and NestedGIN all will require more powerful devices on such instances. 

\begin{landscape}
\begin{table}
\caption{Training and inference time of augmentation methods in graph classification task. All methods are with GIN as backbone. Each result is the time on the whole dataset.}
\centering
\label{testtime}
\setlength{\tabcolsep}{0.6mm}
\scalebox{0.8}{
\begin{tabular}{l|l|cccccccccccc}
\hline\hline 
\textbf{Stage}&\textbf{Method} & \textbf{BZR} & \textbf{COX2} &\textbf{DD} & \textbf{I-BINARY} & \textbf{I-MULTI} & \textbf{MUTAG} & \textbf{NCI1} &  \textbf{NCI109} &  \textbf{PROTS$\_$full}& \textbf{R-BINARY}& \textbf{R-MULTI}& \textbf{ENZYMES}\\
\hline
\multirow{8}{*}{\textbf{Training}}&$\mathcal{G}$-mixup & 0.0281&0.0328& -&0.0548&0.0679&0.0107&0.2547&0.2027& 0.0684& -& -&0.0369\\
&Dummy&0.0186&0.0218&0.2356& -& -&0.0095&0.1556&0.1527&0.0468& -& -&0.0230 \\
&DropGNN &0.3380&0.2419& -&0.2009&0.4198&0.0543&2.2575&2.3114&1.1480	& -& -&0.4901 \\
&rGIN&0.0482&0.0558&0.3594&0.1035&0.1231&0.0213&0.4054&0.3735&0.1172&0.8252&2.1535&0.0590 \\
&NestedGIN&0.1861&0.2056&9.3848&0.8087&0.6975&0.0545&1.3730&1.4143	&0.6841	& -& -&0.2964\\
&GSN-v& 1.7383&1.9700&6.1113&4.6320&3.3124&1.1959&1.3851&1.3868&1.3944&39.6010&50.8742&1.3267 \\													
&LAGIN&0.0492&0.0550&0.3569&0.1009&0.1167&0.0215&0.3647&0.3670&0.1167&0.7567&2.0873&0.0596 \\
&CoS-GIN&0.0870&0.0991&0.7040&0.1719&0.1978&0.0358&0.5852&0.5855&0.1994&1.5373& 4.3893&0.1020 \\
\hline\hline 
\textbf{Stage}&\textbf{Method} & \textbf{BZR} & \textbf{COX2} &\textbf{DD} & \textbf{I-BINARY} & \textbf{I-MULTI} & \textbf{MUTAG} & \textbf{NCI1} &  \textbf{NCI109} &  \textbf{PROTS$\_$full}& \textbf{R-BINARY}& \textbf{R-MULTI}& \textbf{ENZYMES}\\
\hline
\multirow{8}{*}{\textbf{Inference}}&$\mathcal{G}$-mixup & 0.0029 &	0.0029&	- &	0.0042&	0.0054&	0.0023&	0.0122&	0.0125&	0.0049&	- &- & 0.0033 \\
&Dummy & 0.0029&	0.0029	&0.0186	& -& -&	0.0022&	0.0117&	0.0120&	0.0050& -& -&0.0034\\
&DropGNN &0.0188&0.0132&	- &0.0116&0.0184&0.0061&0.1079&0.1161&0.0533& -& -& 0.0247\\
&rGIN&0.0053&0.0054&0.0218&0.0079&0.0099&0.0049&0.0227&0.0216&0.0086&0.0376&0.1185&0.0056\\
&NestedGIN&0.0147&0.0148&0.5884&0.0632&0.0529&0.0056&0.0909&0.0936&0.0404& -& -&0.0207	\\
&GSN-v&0.0088&0.0095&0.0306&0.0160&0.0153&0.0082&0.0246&0.0250&0.0095&0.1121&0.3701&0.0081\\													
&LAGIN&0.0032&0.0035&0.0168&0.0046&0.0045&0.0035&0.0046&0.0046&0.0051&0.0286	&0.0478&0.0037 \\
&CoS-GIN&0.0092&0.0094&0.0410&0.0116&0.0132&0.0079&0.0321&0.0324&0.0129&0.0807 &0.2334&0.0101\\
%\hline\hline 
%\textbf{Stage}&\textbf{Method} & \textbf{BZR} & \textbf{COX2} &\textbf{DD} & \textbf{I-BINARY} & \textbf{I-MULTI} & \textbf{MUTAG} & \textbf{NCI1} &  \textbf{NCI109} &  \textbf{PROTS$\_$full}& \textbf{R-BINARY}& \textbf{R-MULTI}& \textbf{ENZYMES}\\
%\hline
%\multirow{1}{*}{\textbf{Inference with FE}}&CoS-GIN&0.1838&	0.3303&16.2424&1.4627&0.8947&0.4186&1.5588&1.5650&0.8234&107.8466&236.1239&0.4356\\
\hline\hline
\end{tabular}}
\end{table}
\end{landscape}

% the learning of graph specific characteristics. 

% The training and inference time comparison is presented in Appendix. 
%We also examine the applicability of CoS-GNN as GNN backbone in other GNN-based methods, such as Dummy and $\mathcal{G}$-mixup, with the results presented in Appendix showing that CoS-GNN can be used as a GNN backbone in other GNN-based methods to improve their performance successfully.	
\subsection{Employ CoS-GNN as GNN Backbone}
\subsubsection{Combined with GNN-based Methods}
In this section, we examine the applicability of our CoS-GNN as GNN backbone in other GNN-based methods by replacing the GCN of $\mathcal{G}$-mixup and Dummy with CoS-GCN. We omit the results on REDDIT-BINARY and REDDIT-MULTI here because we can not run $\mathcal{G}$-mixup and Dummy on them by our device. The accuracy results of GCN-based and CoS-GCN-based $\mathcal{G}$-mixup and Dummy are reported in Table~\ref{CoS-GCNbasedmixup}. The results show that the CoS-GCN-based $\mathcal{G}$-mixup outperform GCN-based $\mathcal{G}$-mixup on all datasets and the largest improvement can be up to $32\%$. The performance of CoS-GCN-based Dummy method is also better than the GCN-based Dummy on most datasets, achieving up to $25\%$ improvement. The paired signed-rank test indicates that the improvement of CoS-GCN-based $\mathcal{G}$-mixup and Dummy across 10 datasets is significant at $99\%$ and $90\%$ confidence level, respectively. The decrease in accuracy of CoS-GCN-based Dummy on DD and PROTEINS$\_$full might be because that the specific structural statistics augmented on the graphs are influenced and disturbed by the addition of the dummy node. When compared with the results of single CoS-GCN, there are also some improvement brought by CoS-GCN-based $\mathcal{G}$-mixup on BZR ($3.5\%$), MUTAG ($0.5\%$) and NCI1 ($0.5\%$) and by CoS-GCN-based Dummy on ENZYMES ($7.4\%$), which means that the combination with other GNN-based methods can also enhance the power of CoS-GNN. Overall, our CoS-GNN and other GNN-based methods are complementary and can be used as the basic GNN module in GNN-based methods to improve their performance successfully. 

\begin{table}
\caption{Accuracy (mean$\pm$std) results of $\mathcal{G}$-mixup and Dummy using CoS-GCN as the GNN module, with $\mathcal{G}$-mixup and Dummy with GCN as baselines in graph classification. `Different' denotes accuracy improvement ($\uparrow$) or decrease ($\downarrow$) brought by the replacement of CoS-GCN. Both of two methods suffer out of memory on REDDIT-BINARY and REDDIT-MULTI.}
\label{CoS-GCNbasedmixup}
\centering
\setlength{\tabcolsep}{0.6mm}
\scalebox{0.8}{
\begin{tabular}{l|ccc}
\hline\hline
\textbf{Dataset} & \textbf{$\mathcal{G}-$mixup(GCN)}& \textbf{$\mathcal{G}-$mixup(CoS-GCN)}& \textbf{Difference}\\
\hline
BZR & $0.8295\pm0.0188$ & $0.8666\pm0.0471$  & $0.0371\uparrow$ \\
COX2 & $0.7730\pm0.0482$ & $0.7967\pm0.0416$ &$0.0237\uparrow$ \\
DD &  $-$ &$-$ & $-$ \\
%DHFR & $0.6098\pm0.0048$ & $0.7976\pm0.0351$  & $0.1878\uparrow$ \\
I$-$BINARY & $0.7360\pm0.0350$ & $0.7410\pm0.0243$ &$0.0050\uparrow$ \\
I$-$MULTI & $0.5013\pm0.0283$ & $0.5073\pm0.0264$ & $0.0060\uparrow$ \\
MUTAG & $0.7173\pm0.1130$ & $0.8830\pm0.0617$ & $0.1657\uparrow$ \\
NCI1 & $0.5007\pm0.0010$ & $0.8212\pm0.0146$ & $0.3205\uparrow$ \\
NCI109 & $0.5038\pm0.0007$ & $0.7986\pm0.0179$ & $0.2948\uparrow$\\
%PROTEINS  & $0.7251\pm0.0306$  & $0.7583\pm0.0291$ & $0.0332\uparrow$ \\
PROTS$\_$full & $0.7152\pm0.0350$ & $0.7512\pm0.0246$ & $0.0360\uparrow$ \\
ENZYMES & $0.3456\pm0.0435$ & $0.4909\pm0.0490$ & $0.1453\uparrow$\\
\hline
\textbf{p-value} & 0.0039 & - & - \\
\hline\hline
\textbf{Dataset} & \textbf{Dummy(GCN)}&\textbf{Dummy(CoS-GCN)} & \textbf{Difference} \\
\hline
BZR  & $0.8296\pm0.0203$ & $0.8321\pm0.0482$ & $0.0025\uparrow$\\
COX2 & $0.7899\pm0.0509$ & $0.8112\pm0.0654$ & $0.0213\uparrow$\\
DD & $0.7776\pm0.0717$ &  $0.7699\pm0.0433$ & $-0.0077\downarrow$\\
%DHFR  &$0.7499\pm0.0480$ & $0.7526\pm0.0381$ &$0.0027\uparrow$\\
I$-$BINARY  & $-$ & $-$ & $-$  \\
I$-$MULTI & $-$ & $-$ & $-$\\
MUTAG &  $0.7813\pm0.1292$ & $0.8673\pm0.0754$ & $0.0860\uparrow$\\
NCI1  &  $0.6608\pm0.1016$ & $0.8092\pm0.0146$ & $0.1484\uparrow$ \\
NCI109 &$0.5527\pm0.0851$ & $0.8013\pm0.0243$& $0.2486\uparrow$\\
%PROTEINS & $0.7521\pm0.0239$ & $0.7395\pm0.0349$	& $-0.0126\downarrow$  \\
PROTS$\_$full & $0.7557\pm0.0375$ & $0.7556\pm0.0286$& $-0.0001\downarrow$	\\
ENZYMES & $0.4450\pm0.1038$ & $0.6067\pm0.0602$& $0.1617\uparrow$\\
\hline
\textbf{p-value} & 0.0547 & - & - \\
\hline\hline
\end{tabular}
}
\end{table}

\subsubsection{Combined with Other Pooling Methods}
There have been a type of pooling methods that hierarchically extract the graph information~\cite{ying2018hierarchical,zhang2021hierarchical}. We demonstrate that our CoS-GNN structure can be combined with these methods in this section. In each pooling layer, we use the real node features to calculate the pooling criterion and construct the coarsened graph. 

We run our CoS-GCN with a hierarchical pooling--MVPool as the pooling operation to prove that our CoS-GCN can improve the performance of hierarchical pooling. The results of CoS-GCN with MVPool and GCN with MVPool as baseline are reported in Table~\ref{MVPool}, showing that CoS-GCN still can bring improvement on all datasets except COX2 and IMDB-MULTI. The average improvement is $2.02\%$ and the maximal improvement can be up to about $12.28\%$. The paired signed-rank test indicates that the improvement across 12 datasets is significant at $99\%$ confidence level. These results demonstrate that the augmented node and graph structural information also can provide extra useful information while the pooling operation is learning graph structural information. The accuracy decline of CoS-GCN with MVPool on COX2 is very marginal, only $0.01\%$. The $1\%$ drop of CoS-GCN with MVPool on IMDB-MULTI might be because that the structural information in IMDB-MULTI might be limited and the hierarchical learning of MVPool can utilize most structural information in the graph. The feature augmentation in CoS-GCN provides some redundant information to the CoS-GCN-MVPool.

%It is shown that our CoS-GCN still can bring improvement on most datasets, meaning that the augmented node and graph structural information also can provide extra useful information even though the pooling operation is learning structural information meanwhile. 

\begin{table}
\caption{Accuracy (mean$\pm$std) results of GCN and CoS-GCN with MVPool as the readout operation. `Different' denotes accuracy improvement ($\uparrow$) or decrease ($\downarrow$) brought by CoS-GCN compared to GCN.}
\centering
\label{MVPool}
\setlength{\tabcolsep}{0.6mm}
\scalebox{0.8}{
\begin{tabular}{l|ccc}
\hline\hline
\textbf{Dataset} & \textbf{GCN-MVPool}& \textbf{CoS-GCN-MVPool} &\textbf{Difference}\\
\hline
BZR & $0.8273\pm0.0565$ & $0.8345\pm0.0301$ &$0.0072\uparrow$ \\
COX2 & $0.7987\pm0.0351$ & $0.7986\pm0.0430$ &$-0.0001\downarrow$  \\
DD & $0.7750\pm0.0363$ & $0.7962\pm0.0390$ & $0.0212\uparrow$  \\
%DHFR & $0.7591\pm0.0390$ & $0.7618\pm0.0433$ & $0.0027\uparrow$ \\
I$-$BINARY & $0.7280\pm0.0268$ & $0.7350\pm0.0492$ & $0.0070\uparrow$ \\
I$-$MULTI & $0.5180\pm0.0253$ & $0.5080\pm0.0332$ & $-0.0100\downarrow$  \\
MUTAG & $0.7178\pm0.0858$ & $0.8406\pm0.1107$ & $0.1228\uparrow$ \\
NCI1 & $0.7791\pm0.0155$ & $0.8015\pm0.0094$  & $0.0224\uparrow$  \\
NCI109 & $0.7754\pm0.0233$ &  $0.7989\pm0.0198$& $0.0235\uparrow$ \\
%PROTEINS  &  $0.7656\pm0.0315$ & $0.7485\pm0.0329$  	  & $-0.0171\downarrow$\\
PROTS$\_$full & $0.7556\pm0.0360$ &  $0.7664\pm0.0298$& $0.0108\uparrow$ 			\\
R$-$BINARY & $0.9050\pm0.0219$ & $0.9140\pm0.0202$ & $0.0090\uparrow$   \\
R$-$MULTI & $0.5311\pm0.0136$ & $0.5515\pm0.0255$  & $0.0204\uparrow$  \\
ENZYMES & $0.5833\pm0.0516$ & $0.5917\pm0.0455$ & $0.0084\uparrow$  \\
\hline
\textbf{p-value} & 0.0093 & - & - \\
\hline\hline
\end{tabular}
}
\end{table}

\subsection{Enabling Other Down-stream Tasks}
\subsubsection{Graph Anomaly Detection}
We next evaluate the performance of CoS-GNN in anomaly detection, in which the normal graph samples are available for training. It should be noted that this experiment is focused to demonstrate the ability of CoS-GNN in enabling better performance of some popular anomaly detection algorithms, compared to the use of original GNNs, rather than to argue for state-of-the-art anomaly detection performance of CoS-GNN.
%with CoS-GNN as GNN backbone can obtain better performance instead of evaluating the anomaly detection ability of CoS-GNN in this experiment. 
Thus, we use one-class GCN (OCGCN) and GLocalKD~\cite{ma2022deep} as baselines and replace the GCN of GLocalKD with CoS-GCN to examine the ability of CoS-GCN to enable the anomaly detector GLocalKD. Since GLocalKD employs degree information as the node features for plain graphs, which is one of the features we augment in CoS-GNN, we only compare the performance of them on graphs with node attributes. The datasets we use in this experiment are all from TUDataset graph classification benchmark~\cite{Morris+2020}, with the results reported in Table~\ref{anomalydetection}, where FTEIN is short for FRANKENSTEIN.
% The first six results of OCGCN and GLocalKD are from \cite{ma2022deep}. 
It can be observed that CoS-GCN improves the performance of GLocalKD largely on most datasets and the largest improvement can be up to $24.6\%$, which means that the augmented node and graph structural features can also be effectively leveraged via our proposed message passing for improving the detection of anomalies, despite it is an semi-supervised task. The decrease of CoS-GCN on AIDS might be because that AIDS is a rather simple dataset on which the original GLocalKD has obtained an AUC of almost one; CoS-GCN is slightly over-parameterized for such a simple dataset. 
% The untrained abnormal data also can be identified and thus decreases the loss (anomaly score) of abnormalities between teacher and the student networks, leading to some anomalies classified to normal data.

\begin{table}[h]
\centering
\caption{AUC results (mean$\pm$std) of OCGCN, GLocalKD based on GCN, and CoS-GCN-enbaled GLocalKD (CoS-GCN for short) on 12 public attributed graph datasets. `Diff.' denotes AUC improvement ($\uparrow$) or decrease ($\downarrow$) resulted by replacing the GCN backbone with CoS-GCN in GLocalKD.}
\label{anomalydetection}
\setlength{\tabcolsep}{0.6mm}
\scalebox{0.8}{
\begin{tabular}{l|cccc}
\hline\hline
\textbf{Dataset} & \textbf{OCGCN} & \textbf{GLocalKD} & \textbf{CoS-GCN} & \textbf{Diff.} \\
\hline
AIDS & $0.664\pm0.080$ & $0.992\pm0.004$ & $0.948\pm0.008$ & $-0.044\downarrow$\\
BZR & $0.658\pm0.071$ & $0.679\pm0.065$ & $0.804\pm0.068$ & $0.125\uparrow$\\
COX2 & $0.628\pm0.072$ & $0.589\pm0.045$ & $0.665\pm0.050$& $0.076\uparrow$\\
DHFR & $0.495\pm0.080$& $0.558\pm0.030$ & $0.595\pm0.053$& $0.037\uparrow$\\
PROTS$\_$full & $0.718\pm0.036$ & $0.785\pm0.034$ & $0.792\pm0.024$& $0.007\uparrow$\\
ENZYMES  & $0.613\pm0.087$ & $0.636\pm0.061$ & $0.760\pm0.070$& $0.124\uparrow$\\
COIL$-$RAG &$0.629\pm0.210$ & $0.656\pm0.220$ & $0.700\pm0.082$& $0.044\uparrow$\\
Letter$-$high &$0.580\pm0.042$& $0.591\pm0.023$ & $0.655\pm0.071$& $0.064\uparrow$\\
Letter$-$low & $0.616\pm0.168$ & $0.738\pm0.051$ & $0.984\pm0.005$& $0.246\uparrow$\\
Letter$-$med &$0.618\pm0.080$ & $0.662\pm0.062$ & $0.852\pm0.024$& $0.190\uparrow$ \\
FTEIN &$0.550\pm0.031$& $0.547\pm0.019$ & $0.563\pm0.018$& $0.016\uparrow$\\
Synthie &$0.568\pm0.083$& $0.844\pm0.036$ & $0.862\pm0.017$& $0.018\uparrow$\\
\hline\hline
\end{tabular}
}
\end{table}

\subsubsection{Out-of-distribution Generalization}
We evaluate the generalization ability of CoS-GNN on out-of-distribution (OOD) data in this section. This experiment is designed to compare the performance of CoS-GNN with other two message passing neural networks (\ie GCN and GIN) in OOD generalization. These GNNs can be utilized as GNN backbone in various generalization methods to obtain better performance further. 
The datasets we utilize are from GOOD\footnote{https://github.com/divelab/GOOD/}~\cite{gui2022good}. GOOD-Motif is a synthetic dataset designed for structure shifts, GOOD-HIV is a molecular dataset, and GOOD-SST2 is a natural language sentiment analysis dataset. For each dataset, the GOOD benchmark selects one or two domain features (\eg, base and size for GOOD-Motif, scaffold and size for GOOD-HIV, and length for GOOD-SST2) and then applies covariate and concept shift splits per domain to create diverse distribution shifts. Following \cite{gui2022good}, the metric we use for GOOD-HIV is AUC and classification accuracy is used for other datasets. We examine the generalization power of CoS-GNN with the baseline models taken from the GOOD benchmark \cite{gui2022good}. The GNN backbone used in the baselines is GIN. 

The results on the OOD and the in-distribution (ID) validation sets are reported in Table~\ref{generalization}. It can be seen from the results that our model CoS-GCN outperforms the basic GCN on all settings except the one on GOOD-HIV; CoS-GIN gains better performance than GOOD on all settings except GOOD-HIV. This is mainly because that the node and graph structures augmented in our CoS-GNN are more generalizable w.r.t. different shifts of base, size, or length on the three GOOD datasets, while being less generalizable to the scaffold shift, a two-dimensional structural base of a molecule.
% , cause interference to the model.
The especially outstanding performance of CoS-GNN on GOOD-Motif also helps justify this. Each graph in GOOD-Motif is generated by connecting a base graph and a motif, and thus, the structure of base graphs and motifs is highly differentiated. Thus, the augmented structural information of each class enables the structure learning in CoS-GNN to obtain substantially improved OOD generalization performance, when compared with vanilla GNN.  

\begin{table}[h]
\caption{Results of CoS-GNN with two baselines on three OOD datasets. G-X is short for the dataset name GOOD-X.}
\label{generalization}
\centering
\setlength{\tabcolsep}{0.6mm}
\scalebox{0.75}{
\begin{tabular}{l|cccc}
\hline\hline
\multirow{2}{*}{\textbf{G-Motif}} & \multicolumn{4}{c}{\textbf{Base}} \\
\cline{2-5}
 &  \multicolumn{2}{c|}{Covariate} &  \multicolumn{2}{c}{Concept}\\
\hline
\textbf{Accuracy}& \multicolumn{1}{c|}{OOD Validation}& \multicolumn{1}{c|}{ID Validation} & \multicolumn{1}{c|}{OOD Validation}& \multicolumn{1}{c}{ID Validation} \\
\hline 
GCN & $0.321\pm0.000$ & $0.343\pm0.025$ & $0.395\pm0.014$ & $0.382\pm0.017$  \\
GOOD  &  $0.687\pm0.034$ & $0.700\pm0.019$ & $0.814\pm0.006$ & $0.809\pm0.007$ \\
CoS-GCN  & $0.868\pm0.004$ & $0.865\pm0.003$ & $\mathbf{0.934\pm0.000}$ & $\mathbf{0.932\pm0.001}$   \\
CoS-GIN & $\mathbf{0.888\pm0.020}$	&$\mathbf{0.896\pm0.007}$&	$0.931\pm0.001$&	$0.923\pm0.009$\\
\hline\hline
\multirow{2}{*}{\textbf{G-Motif}} & \multicolumn{4}{c}{\textbf{Size}} \\
\cline{2-5}
 & \multicolumn{2}{c|}{Covariate} &  \multicolumn{2}{c}{Concept}\\
\hline
\textbf{Accuracy}& \multicolumn{1}{c|}{OOD Validation}& \multicolumn{1}{c|}{ID Validation}& \multicolumn{1}{c|}{OOD Validation}& \multicolumn{1}{c}{ID Validation} \\
\hline 
GCN &  $0.346\pm0.008$ & $0.350\pm0.003$ & $0.391\pm0.0172$ & $0.385\pm0.018$ \\
GOOD  &  $0.517\pm0.023$ & $0.513\pm0.019$ & $0.708\pm0.006$ & $0.694\pm0.009$\\
CoS-GCN  & $\mathbf{0.863\pm0.043}$ & $\mathbf{0.816\pm0.077}$ & $\mathbf{0.935\pm0.000}$ & $\mathbf{0.933\pm0.002}$  \\
CoS-GIN & $0.598\pm0.070$	&$0.555\pm0.096$&	$0.918\pm0.006$&	$0.898\pm0.014$\\

\hline\hline

\multirow{2}{*}{\textbf{G-HIV}} & \multicolumn{4}{c}{\textbf{Scaffold}} \\
\cline{2-5}
 &  \multicolumn{2}{c|}{Covariate} &  \multicolumn{2}{c}{Concept}\\
\hline
\textbf{AUC}& \multicolumn{1}{c|}{OOD Validation}& \multicolumn{1}{c|}{ID Validation} & \multicolumn{1}{c|}{OOD Validation}& \multicolumn{1}{c}{ID Validation} \\
\hline 
GCN & $0.669\pm0.026$ & $0.676\pm0.016$ & $0.700\pm0.014$ & $0.607\pm0.016$ \\
GOOD  &  $\mathbf{0.696\pm0.020}$ & $0.689\pm0.021$ & $\mathbf{0.723\pm0.010}$ & $\mathbf{0.653\pm0.035}$ \\
CoS-GCN  & $0.690\pm0.017$ & $\mathbf{0.699\pm0.023}$ & $0.708\pm0.009$ & $0.605\pm0.026$ \\
CoS-GIN & $0.684\pm0.021$ & $0.663\pm0.036$ & $0.722\pm0.011$ &	$0.636\pm0.016$ \\

\hline\hline
\multirow{2}{*}{\textbf{G-HIV}} & \multicolumn{4}{c}{\textbf{Size}} \\
\cline{2-5}
&  \multicolumn{2}{c|}{Covariate} &  \multicolumn{2}{c}{Concept} \\
\hline
\textbf{AUC}& \multicolumn{1}{c|}{OOD Validation}& \multicolumn{1}{c|}{ID Validation}& \multicolumn{1}{c|}{OOD Validation}& \multicolumn{1}{c}{ID Validation}\\
\hline
GCN & $0.591\pm0.020$ & $0.580\pm0.012$ & $0.638\pm0.0110$ & $0.533\pm0.009$\\
GOOD  &  $0.600\pm0.029$ & $0.584\pm0.025$ & $0.633\pm0.025$ & $0.448\pm0.029$  \\
CoS-GCN  & $\mathbf{0.607\pm0.019}$ & $\mathbf{0.619\pm0.005}$ & $0.654\pm0.008$ & $0.547\pm0.007$\\
CoS-GIN & $0.585\pm0.029$	&$0.599\pm0.028$&	$\mathbf{0.731\pm0.006}$&	$\mathbf{0.622\pm0.016}$\\
\hline\hline

\multirow{2}{*}{\textbf{G-SST2}} & \multicolumn{4}{c}{\textbf{Length}} \\
\cline{2-5}
 &  \multicolumn{2}{c|}{Covariate} &  \multicolumn{2}{c}{Concept}\\
\cline{1-5}
\textbf{Accuracy}& \multicolumn{1}{c|}{OOD Validation}& \multicolumn{1}{c|}{ID Validation} & \multicolumn{1}{c|}{OOD Validation}& \multicolumn{1}{c}{ID Validation} \\
\cline{1-5} 
GCN & $0.825\pm0.008$ & $0.805\pm0.010$ & $0.724\pm0.012$ & $0.677\pm0.010$ \\
GOOD  &  $0.813\pm0.004$ & $0.778\pm0.011$ & $0.724\pm0.005$ & $0.673\pm0.001$\\
CoS-GCN  &  $\mathbf{0.828\pm0.010}$ & $\mathbf{0.814\pm0.014}$ & $0.730\pm0.007$ & $\mathbf{0.685\pm0.023}$\\
CoS-GIN & $0.822\pm0.012$&	$0.796\pm0.021$&	$\mathbf{0.737\pm0.012}$&	$\mathbf{0.685\pm0.013}$\\
\hline\hline
\end{tabular}
}
\end{table}

\subsection{Robustness w.r.t. Structure Contamination}
Since the data collected in real applications may be with limited/noisy structural information, the performance of our CoS-GNN, which harnesses rich structural information, might be influenced by these contaminated information. In this section, we discuss the impact of limited/noisy structural knowledge on our CoS-GNN. Specifically, we randomly remove $\{1\%, 5\%, 10\%, 15\%, 20\%\}$ edges of the data and compare their results with the results on original data. Our experiment is implemented on GCN backbone. 

The results on NCI1 is displayed in Figure~\ref{strucnoise}. It is obvious that both CoS-GCN and GCN suffer from performance decline due to the edge removal and the decline level of them is similar. Our CoS-GCN always have better performance than GCN under various structural contamination situation. This means that limited/noisy structure brings no more serious effects on the CoS-GCN. This might be because that the structures we augment are in different scales and parts of the extracted features will be infected while others will still be exact. The unaffected structure features can correct the influence of the wrong information brought by the limited/noisy graph structure.

\begin{figure}
    \centering
    \includegraphics[width=5cm]{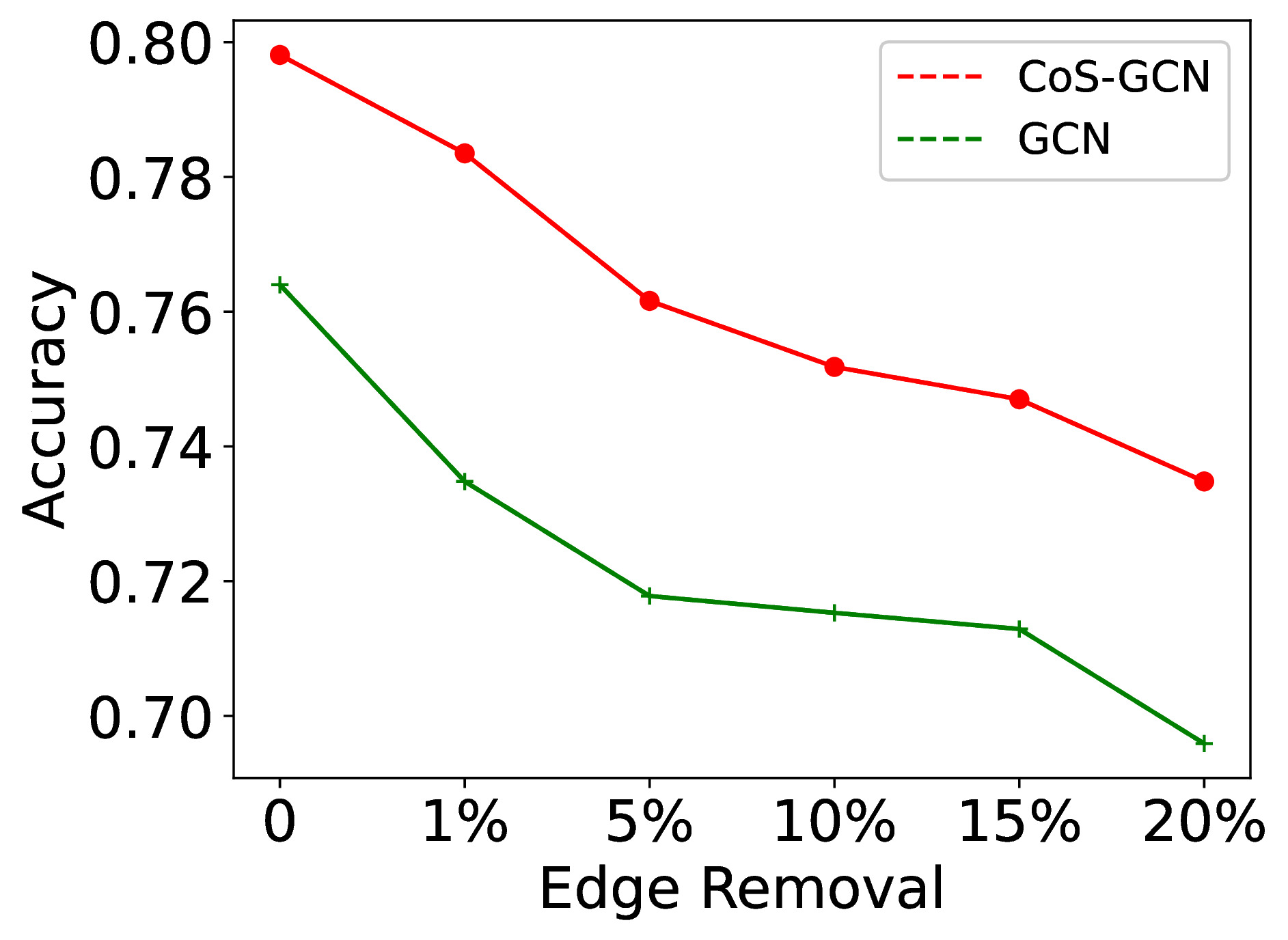}
    \caption{Accuracy performance of CoS-GCN and GCN w.r.t. different structural contamination rates.}
    \label{strucnoise}
\end{figure}

\subsection{Convergence Analysis}
In this section we run an experiment to illustrate the convergence ability of our CoS-GNN. In detail, we run the CoS-GCN on the REDDIT-BINARY dataset and record the loss tendency of training and validation dataset. The result is shown in Figure~\ref{losscurve}. It is obvious that both the training and validation loss will approach stability after a number of epochs. Besides, the early stopping used during training can ensure the model against overfitting and obtaining an excellent result. 

\begin{figure}
    \centering
    \includegraphics[width=5cm]{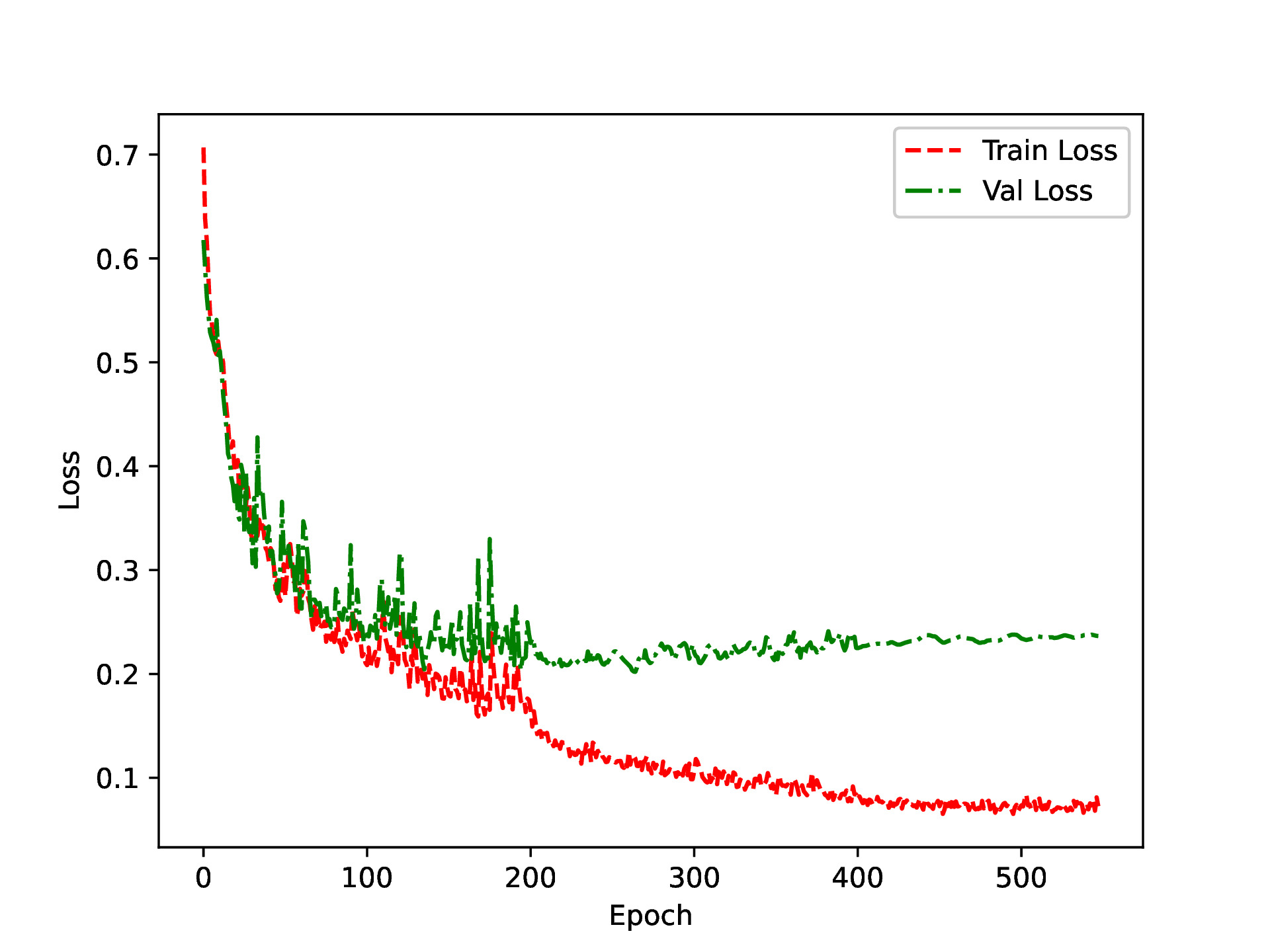}
    \caption{Loss variation tendency of CoS-GCN on the training and validation dataset of REDDIT-BINARY.}
    \label{losscurve}
\end{figure}

\subsection{Ablation Study}\label{subsec:ablation}
\subsubsection{Ablation Study of the Specific Message Passing Scheme}
This section examines the importance of the graph augmentation and the message passing scheme designed in CoS-GNN. All expeirments are based on CoS-GCN. We first evaluate the performance of GCN with original/augmented features as sole input ($\mathcal{V}_{nf}$ for real node feature, $\mathcal{V}_{ns}$ for augmented node structure features, and $\mathcal{V}_{gs}$ for augmented graph structure features), and then combine original and augmented node features by convolution after concatenation (conv$\_$cat($\mathcal{V}_{nf},\mathcal{V}_{ns}$)), concatenation after convolution (cat$\_$conv($\mathcal{V}_{nf},\mathcal{V}_{ns}$))), and convolution with our proposed message passing method (conv($\mathcal{V}_{nf},\mathcal{V}_{ns}$))). Incorporating the augmented graph features to conv($\mathcal{V}_{nf},\mathcal{V}_{ns}$) 
%(\ie, the Graph-level Representation Fusion component) 
leads to the full CoS-GCN. 

The results of our ablation study using the graph classification task are displayed in Table~\ref{ablation}. The paired signed-rank test shows when compared with other ablation parts except conv($\mathcal{V}_{nf},\mathcal{V}_{ns}$), the improvement of CoS-GCN across 12 datasets is significant at $99\%$ confidence level. The enhancement of CoS-GCN than conv($\mathcal{V}_{nf},\mathcal{V}_{ns}$) across all datasets is significant at $85\%$ confidence level. In detail, using node features ($\mathcal{V}_{nf}$) or augmented node/graph structural features ($\mathcal{V}_{ns}$/$\mathcal{V}_{gs}$) solely can achieve good performance, and using $\mathcal{V}_{nf}$ often outperforms $\mathcal{V}_{ns}$ and $\mathcal{V}_{gs}$ on most datasets. This indicates that both the original and augmented features are useful in graph representation learning but the augmented features is limitedly informative.
%This means that the original information in node  is more informative than the augmented structural information. 
% The exclusive application of graph structural knowledge is less powerful, but still can obtain acceptable results. 
The simple concatenation of $\mathcal{V}_{nf}$ and $\mathcal{V}_{ns}$, \ie, $\mathbf{conv\_cat(\mathcal{V}_{nf},\mathcal{V}_{ns})}$ or $\mathbf{cat\_conv(\mathcal{V}_{nf},\mathcal{V}_{ns}))}$, helps improve the performance over the using of them solely, indicating the complementary information gained from the graph augmentation relative to the original node features. Our proposed message passing (convolution) method on top of the real and augmented node features, \ie, $\mathbf{conv(\mathcal{V}_{nf},\mathcal{V}_{ns})}$, further enhances the results substantially, which demonstrates the effectiveness of our proposed message passing in capturing intricate relations that cannot be captured in the vanilla GCN. Lastly, incorporating the augmented graph-level features would lead to the full model CoS-GNN that largely improves $\mathbf{conv(\mathcal{V}_{nf},\mathcal{V}_{ns})}$, demonstrating that the generated global graph structural features are also important for the overall improvement. 
% The addition of graph structural information on the convolution of original and augmented node features (CoS-GCN) achieves the topmost rank among these parts, which means that the augmentation of graph structure also provide useful information. 

\begin{landscape}
\begin{table}[h]
\caption{Results of the ablation study of CoS-GCN in the graph classification task.}
\centering
\label{ablation}
\setlength{\tabcolsep}{0.6mm}
\scalebox{0.8}{
\begin{tabular}{l|ccccccc}
\hline\hline
\textbf{Dataset} & $\mathbf{\mathcal{V}_{nf}}$& $\mathbf{\mathcal{V}_{ns}}$ & $\mathbf{\mathcal{V}_{gs}}$ & $\mathbf{conv\_cat(\mathcal{V}_{nf},\mathcal{V}_{ns})}$ & $\mathbf{cat\_conv(\mathcal{V}_{nf},\mathcal{V}_{ns}))}$&$\mathbf{conv(\mathcal{V}_{nf},\mathcal{V}_{ns})}$ & \textbf{CoS-GCN}
 \\
\hline
BZR & $0.8395\pm0.0280$ & $\mathbf{0.8470\pm0.0284}$ & $0.7877\pm0.0101$ & $0.8075\pm0.0408$ & $0.8298\pm0.0248$
& $\underline{0.8445\pm0.0380}$ & $0.8321\pm0.0361$   \\
COX2 & $\underline{0.8113\pm0.0432}$ & $0.7816\pm0.0081$ & $0.7816\pm0.0081$ & $0.8050\pm0.0432$ & $0.8049\pm0.0444$&$0.8026\pm0.0575$ & $\mathbf{0.8240\pm0.0548}$    \\
DD & $0.7555\pm0.0334$ & $0.7436\pm0.0337$ & $0.7555\pm0.0375$ & $0.7699\pm0.0430$ & $\underline{0.7742\pm0.0431}$ &$0.7725\pm0.0316$  & $\mathbf{0.7784\pm0.0322}$     \\
%DHFR & $0.7354\pm0.0278$ & $0.6337\pm0.0386$ & $0.6098\pm0.0048$ & $\underline{0.7473\pm0.0289}$ & $0.7315\pm0.0543$ &$\mathbf{0.7512\pm0.0349}$ & $0.7433\pm0.0379$     \\
I$-$BINARY & $0.7240\pm0.0364$ & $0.7090\pm0.0559$ & $0.7060\pm0.0518$ & $0.7250\pm0.0686$ & $\underline{0.7410\pm0.0418}
$&$0.7380\pm0.0334$ &  $\mathbf{0.7540\pm0.0516}$     \\
I$-$MULTI & $0.4980\pm0.0240$ & $0.4827\pm0.0320$ & $0.4793\pm0.0327$ & $0.4947\pm0.0275$ & $0.4913\pm0.0257
$&$\mathbf{0.5087\pm0.0237}$ & $\underline{0.5027\pm0.0189}$       \\
MUTAG & $0.7439\pm0.1021$ & $0.8673\pm0.0418$ & $0.8076\pm0.0913$ & $0.8234\pm0.0886$ & $0.8351\pm0.0642
$&$\underline{0.8719\pm0.0661}$ & $\mathbf{0.8775\pm0.0594}$      \\
NCI1 & $0.7847\pm0.0161$  & $0.6895\pm0.0182$ & $0.6343\pm0.0145$ & $0.7903\pm0.0168$ & $0.7888\pm0.0188
$&$\underline{0.8083\pm0.0113}$ & $\mathbf{0.8163\pm0.0134}$       \\
NCI109 & $0.7686\pm0.0263$ & $0.6991\pm0.0253$ & $0.6293\pm0.0245$ & $0.7737\pm0.0177$ & $0.7727\pm0.0285
$&$\underline{0.7991\pm0.0228}$ & $\mathbf{0.8013\pm0.0207}$     \\
%PROTEINS  & $\mathbf{0.7503\pm0.0250}$  & $0.7152\pm0.0315$ & $0.7242\pm0.0301$ &  $0.7413\pm0.0340$ & $\underline{0.7467\pm0.0387}$&$0.7215\pm0.0250$ & $0.7386\pm0.0284$      	  \\
PROTS$\_$full & $0.7493\pm0.0284$  & $0.7278\pm0.0306$ & $0.7296\pm0.0232$ & $\underline{0.7583\pm0.0337}$ &$0.7521\pm0.0294$ & $\mathbf{0.7619\pm0.0267}$	 & $0.7574\pm0.0279$		\\
R$-$BINARY & $0.8995\pm0.0241$ & $\underline{0.9115\pm0.0249}$ & $0.8295\pm0.0203$ & $0.9095\pm0.0268$ & $0.9080\pm0.0268$&$0.9110\pm0.0250$ & $\mathbf{0.9170\pm0.0215}$    \\
R$-$MULTI & $0.5333\pm0.0129$  & $0.5443\pm0.0181$ & $0.5031\pm0.0265$ & $0.5425\pm0.0149$ & $0.5425\pm0.0208$&$\underline{0.5535\pm0.0114}$ & $\mathbf{0.5535\pm0.0275}$    \\
ENZYMES & $0.4500\pm0.0679$ & $0.2783\pm0.0325$ &$0.2750\pm0.0651$ & $0.4317\pm0.0669$ & $\underline{0.5017\pm0.0626}$&$0.4617\pm0.0738$ & $\mathbf{0.5333\pm0.0641}$       \\
\hline
% \textbf{p-value} & 0.0015 & 0.0034 & 0.0004 & 0.0009 & 0.0004 & 0.1289 & - \\
\textbf{Rank} & 4.7 & 4.9 & 6.6& 3.9& 3.8& $\underline{2.4}$& $\mathbf{1.5}$\\
\textbf{p-value} & 0.0015& 0.0034& 0.0005& 0.0010& 0.0005& 0.1289
& $-$\\
\hline\hline
\end{tabular}
}
\end{table}	
\end{landscape}

%We also performed experiments to evaluate the importance of different augmented structural features. Due to the page limitation, the results are presented in Appendix.
\begin{landscape}
\begin{table}[h]
\caption{Efficiency of the augmented node features in graph classification.}
\centering
\label{ablationnode}
\setlength{\tabcolsep}{0.3mm}
\scalebox{0.75}{
\begin{tabular}{l|ccccc|ccccc|c}
\hline\hline 
& \multicolumn{5}{c|}{\textbf{Separated features}} &  \multicolumn{5}{c|}{\textbf{Combined features}} & \textbf{Completed} \\
\hline
& \multicolumn{3}{c|}{\textbf{Structural}} &  \multicolumn{2}{c|}{\textbf{Quantized}} &  \multicolumn{3}{c|}{\textbf{Structural}} & \multicolumn{1}{c|}{\textbf{Quantized}} & \multicolumn{1}{c|}{\textbf{Structural}} &  \textbf{Completed} \\
\hline
\textbf{Dataset} & \textbf{w/o Dg}  & \textbf{w/o Tri}  & \textbf{w/o CK}  & \textbf{w/o TCo}  & \textbf{w/o SCo} &   \textbf{w/o TCK}  & \textbf{w/o DT}  & \textbf{w/o DCK}
&\textbf{w/o n$\_$quant} &\textbf{w/o n$\_$sub} &\textbf{CoS-GCN}\\
\hline
\textbf{BZR} &0.847$\pm$0.048 &0.832$\pm$0.043 &0.857$\pm$0.039 &0.835$\pm$0.052 &0.857$\pm$0.048 &0.842$\pm$0.039 &0.845$\pm$0.053 &0.857$\pm$0.042 &0.815$\pm$0.064&0.832$\pm$0.052& 0.832$\pm$0.036\\
\textbf{COX2} &0.818$\pm$0.049 &0.829$\pm$0.043 &0.814$\pm$0.043 &0.827$\pm$0.045 &0.805$\pm$0.072 &0.807$\pm$0.051 &0.827$\pm$0.062 &0.816$\pm$0.041& 0.803$\pm$0.054&0.822$\pm$0.056&0.824$\pm$0.055\\
\textbf{DD}& 0.774$\pm$0.027& 0.757$\pm$0.028& 0.750$\pm$0.021& 0.753$\pm$0.027& 0.756$\pm$0.032& 0.750$\pm$0.035& 0.743$\pm$0.036& 0.750$\pm$0.024& 0.777$\pm$0.030&0.770$\pm$0.030&0.778$\pm$0.032\\
\textbf{I-BINARY}& 0.749$\pm$0.040& 0.738$\pm$0.041& 0.726$\pm$0.042& 0.747$\pm$0.026& 0.734$\pm$0.036& 0.731$\pm$0.050& 0.732$\pm$0.041& 0.726$\pm$0.049& 0.750$\pm$0.053&0.726$\pm$0.034&0.754$\pm$0.052\\
\textbf{I-MULTI}& 0.473$\pm$0.039& 0.498$\pm$0.019& 0.469$\pm$0.031& 0.489$\pm$0.029& 0.486$\pm$0.034& 0.495$\pm$0.024& 0.492$\pm$0.022& 0.458$\pm$0.025& 0.493$\pm$0.033&0.503$\pm$0.024&0.503$\pm$0.019\\
\textbf{MUTAG}& 0.861$\pm$0.093& 0.841$\pm$0.070& 0.803$\pm$0.082& 0.851$\pm$0.066& 0.872$\pm$0.068& 0.808$\pm$0.077& 0.851$\pm$0.057& 0.781$\pm$0.103& 0.862$\pm$0.060&0.835$\pm$0.077& 0.878$\pm$0.059\\
\textbf{NCI1}& 0.780$\pm$0.019& 0.782$\pm$0.023& 0.782$\pm$0.021& 0.793$\pm$0.016& 0.791$\pm$0.017& 0.775$\pm$0.014& 0.783$\pm$0.022& 0.751$\pm$0.012& 0.816$\pm$0.017& 0.817$\pm$0.013& 0.816$\pm$0.013\\
\textbf{NCI109}& 0.777$\pm$0.025& 0.783$\pm$0.019& 0.767$\pm$0.024& 0.776$\pm$0.018& 0.780$\pm$0.023& 0.760$\pm$0.029& 0.772$\pm$0.025& 0.722$\pm$0.030& 0.796$\pm$0.024& 0.795$\pm$0.022& 0.801$\pm$0.021\\
\textbf{PROTS$\_$full}& 0.750$\pm$0.039& 0.749$\pm$0.034& 0.753$\pm$0.039& 0.758$\pm$0.038& 0.757$\pm$0.039& 0.758$\pm$0.042& 0.750$\pm$0.031& 0.748$\pm$0.036& 0.755$\pm$0.025&0.761$\pm$0.026& 0.757$\pm$0.028\\
\textbf{R-BINARY}& 0.915$\pm$0.013& 0.920$\pm$0.009& 0.889$\pm$0.025& 0.915$\pm$0.018& 0.905$\pm$0.018& 0.890$\pm$0.022& 0.918$\pm$0.017& 0.902$\pm$0.025& 0.906$\pm$0.018&0.895$\pm$0.033& 0.917$\pm$0.022\\
\textbf{R-MULTI}& 0.541$\pm$0.015& 0.554$\pm$0.014& 0.523$\pm$0.018& 0.558$\pm$0.021& 0.546$\pm$0.019& 0.528$\pm$0.021& 0.551$\pm$0.013& 0.549$\pm$0.022& 0.557$\pm$0.009&0.550$\pm$0.022& 0.554$\pm$0.028\\
\textbf{ENZYMES}& 0.490$\pm$0.048& 0.568$\pm$0.057& 0.557$\pm$0.074& 0.523$\pm$0.043& 0.545$\pm$0.066& 0.510$\pm$0.051& 0.528$\pm$0.050& 0.503$\pm$0.045& 0.520$\pm$0.059&0.535$\pm$0.072& 0.533$\pm$0.064\\
\hline
\textbf{Rank} & 6.3 & 4.6& 7.8& 4.8 & 5.4 & 7.9 & 5.9 & 8.8 & 5.0& 4.9 & 2.8 \\
\textbf{p-value}& 0.0093 &0.1387 &0.0093 &0.0400 &0.0986 &	0.0034 &0.0220 &	0.0049 &0.0049 &0.0488 &$-$\\
\hline\hline
\end{tabular}}
\end{table}					
\end{landscape}

\begin{landscape}
\begin{table}[h]
\caption{Efficiency of the augmented graph features in graph classification.}
\centering
\label{ablationgraph}
\setlength{\tabcolsep}{0.3mm}
\scalebox{0.75}{
\begin{tabular}{l|ccccc|ccccc|c}
\hline\hline 
& \multicolumn{5}{c|}{\textbf{Separated features}} &  \multicolumn{5}{c|}{\textbf{Combined features}} & \textbf{Completed} \\
\hline
& \multicolumn{3}{c|}{\textbf{Structural}} &  \multicolumn{2}{c|}{\textbf{Quantized}} &  \multicolumn{3}{c|}{\textbf{Structural}} & \multicolumn{1}{c|}{\textbf{Quantized}} & \multicolumn{1}{c|}{\textbf{Structural}} &  \textbf{Completed} \\
\hline
\textbf{Dataset} & \textbf{w/o Tri} & \textbf{w/o Cli}& \textbf{w/o Bri} & \textbf{w/o ClCo} & \textbf{w/o Effi} &\textbf{w/o TBri} & \textbf{w/o ClBri} & \textbf{w/o TrCl} & \textbf{w/o g$\_$quant} & \textbf{w/o g$\_$sub} & \textbf{CoS-GCN}\\
\hline
\textbf{BZR} &0.850$\pm$0.056 &0.852$\pm$0.032 &0.857$\pm$0.043 &0.840$\pm$0.053 &0.840$\pm$0.050 &0.845$\pm$0.037 &0.847$\pm$0.056 &0.845$\pm$0.068&0.785$\pm$0.023&0.830$\pm$0.051 & 0.832$\pm$0.036\\
\textbf{COX2} & 0.812$\pm$0.056& 0.820$\pm$0.052& 0.827$\pm$0.070& 0.812$\pm$0.061& 0.824$\pm$0.048& 0.822$\pm$0.050& 0.799$\pm$0.068& 0.807$\pm$0.055& 0.784$\pm$0.011&0.798$\pm$0.057&0.824$\pm$0.055\\
\textbf{DD}&0.759$\pm$0.031&0.769$\pm$0.026&0.764$\pm$0.029&0.766$\pm$0.030&0.764$\pm$0.020&0.768$\pm$0.034&0.773$\pm$0.034&0.761$\pm$0.029&	0.745$\pm$0.044&0.784$\pm$0.040&0.778$\pm$0.032\\
\textbf{I-BINARY}& 0.744$\pm$0.040& 0.730$\pm$0.035& 0.738$\pm$0.035& 0.733$\pm$0.026& 0.744$\pm$0.042& 0.736$\pm$0.036& 0.731$\pm$0.034& 0.723$\pm$0.036 &0.567$\pm$0.059 & 0.736$\pm$0.051&0.754$\pm$0.052\\
\textbf{I-MULTI} &0.487$\pm$0.028& 0.480$\pm$0.024& 0.475$\pm$0.035& 0.475$\pm$0.024& 0.471$\pm$0.024& 0.477$\pm$0.025& 0.475$\pm$0.028& 0.499$\pm$0.025&0.366$\pm$0.033&0.501$\pm$0.026&0.503$\pm$0.019\\
\textbf{MUTAG} & 0.830$\pm$0.065 & 0.856$\pm$0.067 & 0.856$\pm$0.082 & 0.867$\pm$0.064 & 0.878$\pm$0.058 & 0.825$\pm$0.062 & 0.856$\pm$0.078 & 0.846$\pm$0.060& 0.856$\pm$0.076&0.861$\pm$0.083&0.878$\pm$0.059\\
\textbf{NCI1}  & 0.796$\pm$0.022 & 0.794$\pm$0.013 & 0.795$\pm$0.025 & 0.753$\pm$0.081 & 0.798$\pm$0.017 & 0.790$\pm$0.030 & 0.787$\pm$0.015 & 0.787$\pm$0.027& 0.784$\pm$0.018&0.810$\pm$0.016&0.816$\pm$0.013\\
\textbf{NCI109} & 0.783$\pm$0.018& 0.779$\pm$0.016& 0.774$\pm$0.025& 0.772$\pm$0.025& 0.776$\pm$0.028& 0.782$\pm$0.015& 0.781$\pm$0.017& 0.782$\pm$0.024&  0.772$\pm$0.024&0.799$\pm$0.016&0.801$\pm$0.021\\
\textbf{PROTS$\_$full}& 0.766$\pm$0.046& 0.741$\pm$0.044& 0.750$\pm$0.039& 0.757$\pm$0.042& 0.769$\pm$0.041& 0.762$\pm$0.036& 0.754$\pm$0.044&0.745$\pm$0.045&0.742$\pm$0.020&0.748$\pm$0.029&0.757$\pm$0.028\\
\textbf{R-BINARY}&  0.915$\pm$0.014& 0.915$\pm$0.017& 0.914$\pm$0.017& 0.917$\pm$0.021& 0.920$\pm$0.015& 0.909$\pm$0.029& 0.917$\pm$0.022& 0.917$\pm$0.022& 0.757$\pm$0.045& 0.911$\pm$0.024& 0.917$\pm$0.022\\
\textbf{R-MULTI}& 0.551$\pm$0.021& 0.563$\pm$0.024& 0.547$\pm$0.017& 0.543$\pm$0.024& 0.551$\pm$0.019& 0.540$\pm$0.026& 0.560$\pm$0.015& 0.556$\pm$0.019&  0.250$\pm$0.053&0.552$\pm$0.015&0.554$\pm$0.028\\
\textbf{ENZYMES}& 0.520$\pm$0.049&0.547$\pm$0.067&0.537$\pm$0.044&0.523$\pm$0.085&0.535$\pm$0.041&0.552$\pm$0.074&0.535$\pm$0.068&0.563$\pm$0.077& 0.173$\pm$0.039&0.525$\pm$0.040&0.533$\pm$0.064\\
\hline
\textbf{Rank}&5.5&5.3&5.5&6.8&4.4&6.0&5.4&5.9&10.3&5.5&2.9\\
\textbf{p-value} & 0.0278&0.0669&0.0542&0.0039&0.168&0.0679	&0.0420&	0.0977&0.0005&0.0068 &$-$\\
\hline\hline
\end{tabular}}
\end{table}						
\end{landscape}

\subsubsection{Ablation Study of the Augmented Features}
In this section, we evaluate the effect of each augmented features on the final performance of CoS-GNN. We divided the augmented features into two categories, \ie, one is the characteristics of some specific substructures, and another is some quantized values to measure structural properties of the node/graph. Then we remove each feature and their combinations in each category separately and compare the classification results with our CoS-GNN. The removal of node and graph features are implemented separately. The GNN backbone we use here is CoS-GCN. 

Firstly, we delete degree, triangle, clique and k-core (denoted by w/o Dg, w/o Tri, w/o CK respectively) and then remove their combinations, \ie, degree and triangle; degree, clique and k-core; triangle, clique and k-core; all the characteristics (shortened to w/o DT, w/o DCK, w/o TCK, w/o n$\_$sub). We also remove quantized values -- triangle clustering coefficient, square clustering coefficient and their combination (written as w/o TCo, w/o SCo and w/o n$\_$quant respectively). The results are reported in Table~\ref{ablationnode}. It is obvious that the removal of augmented features might cause better performance on some specific datasets but will results in decline on many other datasets, leading to a clear decline in the overall performance. Although our CoS-GCN still ranks first on the overall performance, yhe paired signed-rank test indicates that the performance drop of models with part of augmented features across 12 datasets is significant at $85\%$ to $99\%$ confidence level. Deletion of degree and clique and k-core characteristics respectively and their combinations often lead to worse performance, indicating their effect in the full CoS-GCN. Omitting all the node structural characteristics performs better than removing part of them on some datasets and this might be because that the remaining structural characteristics increase the similarity among data.

Later, we delete augmented graph features sequentially (\ie, w/o Tri, w/o Cli, w/o Bri, w/o ClCo and w/o Effi stand for removing triangle, clique, bridge numbers, average clustering coefficient and average local and global efficiency, respectively; w/o TBri, w/o ClBri and w/o TrCl denote deleting triangle and bridge numbers, clique and bridge numbers and triangle and clique numbers; w/o g$\_$quant and w/o g$\_$sub means removing the quantized values and graph substructural statistics, respectively). The results are shown in Table~\ref{ablationgraph}. The improvement of our CoS-GCN over the competing methods across the datasets is significant at $80\%$ to $99\%$ confidence level. The removal of triangle, clique, bridge and their combination knowledge results in similar overall performance, which might be because that each feature contributes to the performance of CoS-GNN differently in different dataset. The deletion of average clustering coefficient or all quantized values has larger effect on the final performance, which indicates that the average clustering coefficient information is more discriminative. In summary, the graph-level substructural characteristics are also beneficial in our CoS-GNN since the removal of them leads to a clear decline of the overall performance of CoS-GNN. 

%It can be seen that the augmented features we harness can be divided into two categories, \ie, one is the characteristics of some specific substructures, and another is some quantized values to measure structural properties of the node/graph. In this section, we run an experiment to examine the effect of each categories in the performance of CoS-GNN. 
%We first remove the quantized values in augmented node structural features (w/o n$\_$quant) and then keep these features but remove the substructural statistics (w/o n$\_$sub). We also run the experiment by removing the quantized values in augmented graph structural features (w/o g$\_$quant) while keep these features but remove the graph substructural statistics (w/o g$\_$sub). The GNN backbone we use here is CoS-GCN. The results are displayed in Table~\ref{ablationpart}. We can see that removing one type of augmented structural information has negative influence on most datasets, meaning that each type of augmented features is important in CoS-GNN. 

\section{Conclusion}
In this work, we propose a collective structure knowledge-augmented graph neural network (CoS-GNN) to enhance the expressive power of conventional message passing neural networks. The augmented node and graph features carry important and generalizable structural knowledge, which is tapped by our proposed message passing mechanism to integrate the original and augmented graph knowledge, resulting in graph representations with significantly improved expressiveness. This is justified by extensive experiments 
% show that our CoS-GNN largely outperform other SOTA competitors 
in various down-stream tasks, including graph classification, anomaly detection, and OOD generalization.

\section{Acknowledgments}
In this work R. Ma and L. Chen are supported by ARC DP210101347.

\end{document}